\documentclass[accepted]{uai2026}

\usepackage[american]{babel}

\usepackage{natbib} 
\usepackage{mathtools} 
\usepackage{booktabs} 
\usepackage{tikz} 

\usepackage{mathtools}
\usepackage{amsthm, amssymb}
\usepackage{amsmath} 
\usepackage{tikz-cd}
\usepackage{bm}  
\usepackage{graphics} 
\usepackage{wrapfig}
\usepackage{subcaption}
\usepackage{color}
\usepackage{booktabs, multirow, float, threeparttable}

\usepackage{stfloats} 

\usepackage{array}      
\usepackage{adjustbox}  

\usepackage{thmtools,thm-restate}
\theoremstyle{thmstyleone}

\theoremstyle{thmstyletwo}

\theoremstyle{thmstylethree}

\usepackage{algpseudocode}
\usepackage{algorithm}

\usepackage{subcaption}  
\usepackage{float}       

\graphicspath{{Figures/}}

\newcommand{\trsp}{{\scriptscriptstyle{\mathsf{T}}}}

\newcommand{\euclideanspace}{\mathbb{R}}
\newcommand{\manifold}{\mathcal{M}}
\newcommand{\tangentspace}[1]{\mathcal{T}_{#1}\mathcal{M}}
\newcommand{\simplex}[1]{\Delta^{#1}}
\newcommand{\sphere}[1]{\mathbb{S}^{#1}}
\newcommand{\tangentsimplex}[2]{\mathcal{T}_{#2}\simplex{#1}}
\newcommand{\expmapblank}[1]{\text{Exp}_{#1}}  
\newcommand{\logmapblank}[1]{\text{Log}_{#1}}  
\newcommand{\expmap}[2]{\expmapblank{#1}(#2)}  
\newcommand{\grad}{\mathrm{grad}}

\definecolor{gaboyellow}{HTML}{DDCC77}
\definecolor{gabored}{HTML}{CC6677}
\definecolor{bopurple}{HTML}{332288}
\definecolor{boblue}{HTML}{88CCEE}
\definecolor{dodgerblue}{rgb}{0.12, 0.565, 1.0}
\definecolor{middlegreen}{rgb}{0.7, 0.9, 0.4}
\definecolor{gaboyellowdark}{rgb}{0.66, 0.6, 0.26}

\DeclareRobustCommand{\blackdottedline}{\raisebox{2pt}{\tikz{\draw[black,dotted,line width = 1.1pt](0,0) -- (3mm,0);}}}
\DeclareRobustCommand{\aogaboline}{\raisebox{2pt}{\tikz{\draw[gaboyellow,solid,line width = 1.1pt](0,0) -- (3mm,0);}}}
\DeclareRobustCommand{\amogaboline}{\raisebox{2pt}{\tikz{\draw[gabored,solid,line width = 1.1pt](0,0) -- (3mm,0);}}}
\DeclareRobustCommand{\spheuboline}{\raisebox{2pt}{\tikz{\draw[bopurple,solid,line width = 1.1pt](0,0) -- (3mm,0);}}}
\DeclareRobustCommand{\borisline}{\raisebox{2pt}{\tikz{\draw[boblue,solid,line width = 1.1pt](0,0) -- (3mm,0);}}}

\DeclareRobustCommand{\bluecircle}{\tikz{ \filldraw[color=black, fill=dodgerblue, thick](0,0) circle (.08);}}
\DeclareRobustCommand{\greencircle}{\tikz{ \filldraw[color=black, fill=middlegreen, thick](0,0) circle (.08);}}

\title{Information Theoretic Bayesian Optimization over the Probability Simplex}

\author[1, 2]{Federico Pavesi}
\author[2]{Antonio Candelieri}
\author[1]{Noémie Jaquier}
\affil[1]{%
    Department of Robotics, Perception and Learning\\
    KTH Royal Institute of Technology\\
    Stockholm, Sweden
}
\affil[2]{%
    Department of Economics, Management and Statistics\\
    University of Milan, Bicocca\\
    Milan, Italy
}

\begin{document}
\maketitle

\begin{abstract}
  Bayesian optimization is a data-efficient technique that has been shown to be extremely powerful to optimize expensive, black-box, and possibly noisy objective functions. 
  Many applications involve optimizing probabilities and mixtures which naturally belong to the probability simplex, a constrained non-Euclidean domain defined by non-negative entries summing to one. 
  This paper introduces $\alpha$-GaBO, a novel family of Bayesian optimization algorithms over the probability simplex. Our approach is grounded in information geometry, a branch of Riemannian geometry which endows the simplex with a Riemannian metric and a class of connections. Based on information geometry theory, we construct Matérn kernels that reflect the geometry of the probability simplex, as well as a one-parameter family of geometric optimizers for the acquisition function. We validate our method on benchmark functions and on a variety of real-world applications including mixtures of components, mixtures of classifiers, and a robotic control task, showing its increased performance compared to constrained Euclidean approaches.
\end{abstract}

\section{Introduction}\label{sec:intro}
Bayesian optimization (BO)~\citep{jones1998:BO, mockus2005:BO, garnett_bayesopt} is a prominent approach rooted in Bayesian probabilistic numerics~\citep{hennig2015probabilistic} for optimizing complex, expensive-to-evaluate black-box functions. 
The main idea of BO is to construct a probabilistic surrogate model of the objective function and to guide the search for the optimum from a decision-theoretic perspective based on the information obtained from point evaluations. Besides its universal approximation properties and convergence guarantees, BO has demonstrated remarkable practical success across a wide range of applications, including hyperparameter tuning in machine learning~\citep{wu2019:hyperparameterBO}, robotics~\citep{jaquier20a, jaquier22a}, experimental design~\citep{greenhill2020:experimentaldesignBO}, and materials and chemical discovery~\citep{zuo2021:materialBO}, with extensive empirical evidence reported for both continuous Euclidean search spaces~\citep{shahriari2015:surveyBO} and discrete domains~\citep{deshwal2021:combinatorialBO,gonzalez2024:surveydiscreteBO}. 

\begin{figure}[t]
  \centering
  \includegraphics[width=.80\columnwidth]{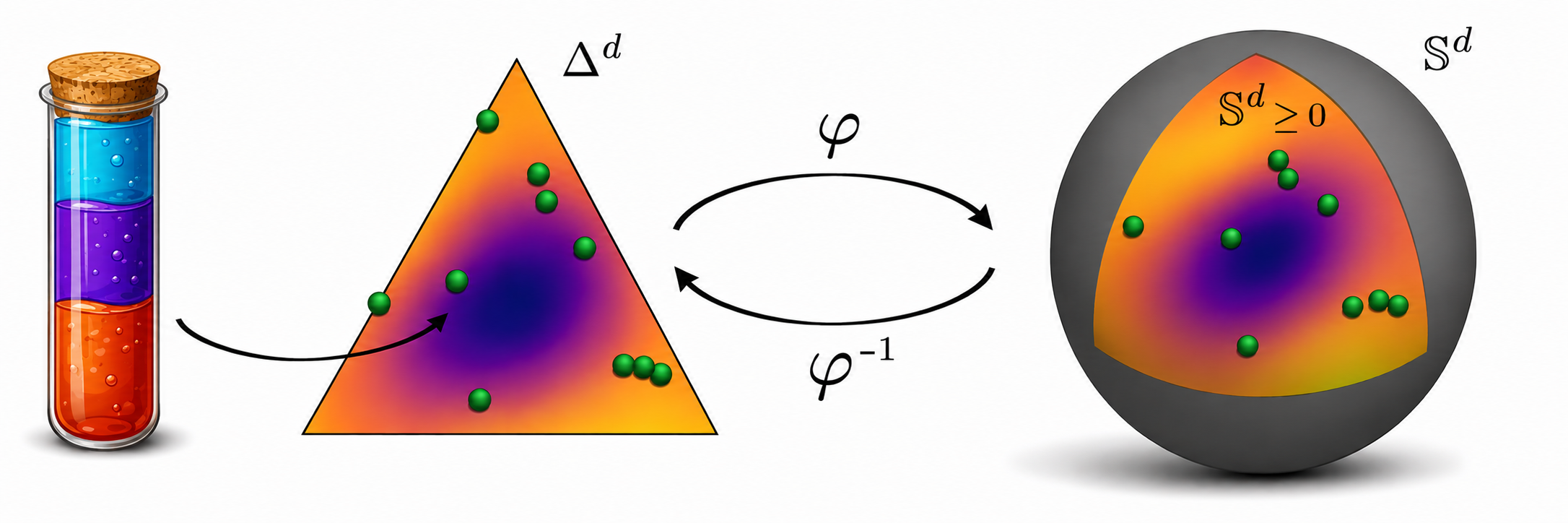}

  \caption{$\alpha$-GaBO leverages the sphere map $\varphi$, which establishes an isometry between the probability simplex $\simplex{d}$ and the positive orthant $\sphere{d}_{\geq0}$ of the sphere. BO on the simplex is performed via equivalent representations on the sphere.}
  \label{fig:teaser}
\end{figure}

The performance and scalability of BO can be significantly enhanced by introducing inductive bias, specifically prior information about the search space. In particular, geometry-aware Bayesian optimization (GaBO) exploits the geometry of the search space to seek optimum on Riemannian manifolds~\citep{jaquier20a, jaquier20b, jaquier22a}. The key ingredients of GaBO are Riemannian kernels and acquisition function optimizers that capture the manifold's underlying geometric structure. 
The probability simplex, composed of vectors with non-negative entries summing to one, is a widely-encountered yet overlooked geometric domain. It naturally arise as the search space when optimizing probabilities and mixtures, e.g., in mixture of experts~\citep{masoudnia2014:mixtureexpertssurvey}, portfolio optimization~\citep{kalai2002efficient}, chemical mixture design~\citep{soares2007statistical}, robotics~\citep{modugno2016:softtaskpriorities,Jaquier22:SeqBlendSkills}, and population dynamics~\citep{bomze2002regularity}.
An initial attempt at geometry-aware BO on the probability simplex, so-called BORIS, was introduced by~\citet{candelieriBO}, who proposed a squared exponential kernel based on the Wasserstein distance. In their implementation, \citet{candelieriBO2} approximate the Wasserstein distance with the standard Euclidean norm, so that BORIS effectively reverses to a constrained Euclidean BO that ignores the intrinsic geometry of the domain. This approach was shown to yield to suboptimal performances for other non-Euclidean search spaces~\citep{jaquier20a, jaquier20b, jaquier22a}. From a theoretical perspective, a fully rigorous, geometry-aware Bayesian optimization framework specifically tailored to the probability simplex remains, to the best of our knowledge, an open area in the literature.\looseness-1

In this paper, we introduce $\alpha$-GaBO, a family of geometry-aware Bayesian optimization algorithms on the probability simplex building on its information-theoretic Riemannian structure (see Sec.~\ref{sec:Background} for a background). First, we leverage the Fisher-Rao metric to establish an isometry between the probability simplex and a subset of the unit hypersphere (see Fig.~\ref{fig:teaser}), allowing us to use well-established Riemannian spherical kernels~\citep{borovitskiy2020matern} on the probability simplex (Sec.~\ref{subsec:simplex-kernels}). Second, we exploit conjugate connections on information manifolds to build a one-parameter family of acquisition function optimizers accounting for the geometry of the probability simplex (Sec.~\ref{subsec:simplex-acqfct}). Focusing on two parameter values, we recover the Levi-Civita and exponential connections, both of which lead to closed-form expressions for the Riemannian operations required during optimization. We test $\alpha$-GaBO on different benchmark functions and real-world applications, including chemical mixtures, mixture of classifiers, and robotic multi-task control (Sec.~\ref{sec: experiments}). Our results show that $\alpha$-GaBO outperforms BORIS and other constrained Euclidean BO frameworks.
The source code is available at{
\href{https://github.com/FedericoPavesi/alphagabosimplex}{https://github.com/FedericoPavesi/alphagabosimplex}.

\section{Related work}
\label{sec: related work}

This section reviews the literature focused on optimizing different classes of functions over the probability simplex. 

\paragraph{Optimization of convex functions.}
Historically, research on optimization over the probability simplex has focused primarily on convex functions. The well-known Frank-Wolfe method~\citep{frankwolfe} exploits the convexity of the optimization domain to reduce the function value while explicitly remaining within the probability simplex~\citep{jaggi2013:revisitingfrankewolfe}. However, this method is known to display oscillatory behavior near the optima, which significantly slows down convergence. Faster convergence is achieved by adding a so-called ``away-step'' to alleviate oscillations~\citep{guelat1986}. In contrast, projected gradient descent~\citep{galantai2003projectors} consists of a simple gradient step followed by a projection onto the domain, i.e., the probability simplex.  
Despite ensuring feasibility, none of the aforementioned approaches accounts for the natural geometry of the probability simplex. In the broader context of optimization on manifolds, accounting for the geometry of the domain was shown to deliver computationally-superior algorithms~\citep{absil2008optimization,Boumal22:RiemannOpt,Weber23:RiemFW}. 
Recently, ~\citet{chock2025} proposed to map the probability simplex to the positive orthant of the unit sphere, on which the objective function is optimized via Riemannian optimization with proven fast convergence rate. In this paper, we similarly take advantage of the so-called sphere map to optimize black-box functions over the probability simplex on the sphere as a domain with simple, well-known geometry.

\paragraph{Optimization of black-box functions.}


Closely related to our work is BORIS, the BO framework from~\citet{candelieriBO} that defines a squared-exponential kernel over the probability simplex based on Wasserstein geometry. In practice, BORIS endows the probability simplex with the discrete metric, which reduces the Wasserstein distance to a closed-form expression equivalent to a Euclidean distance~\citep{candelieriBO2}. This simplification enables the use of the standard Euclidean squared-exponential kernel, effectively treating the probability simplex as a Euclidean space. Moreover, the Riemannian optimization of the acquisition function is hindered by the inherent complexity of Wasserstein geometry, which does not allow for closed-form Riemannian gradients and operations. Therefore, BORIS relies on Euclidean geometry, which only locally approximates the probability simplex geometry, and maximizes the acquisition function via constrained Euclidean optimization.



\section{Background}
\label{sec:Background}

This section introduces the fundamental concepts at the base of our approach. We refer to App.~\ref{appendix:IG} for an extended background on the probability simplex. 

\subsection{Riemannian manifolds}
\label{subsec:riemannmanifolds}
A finite dimensional topological manifold is a subset of the Euclidean space that can be covered by a countable collection of open sets, each diffeomorphic to a subset of a Euclidean space. The tangent space $\tangentspace{\bm{x}}$ is the space of derivations at a point $\bm{x} \in \manifold$. A vector field $X \in \mathfrak{X}(\manifold)$ is a continuous map that assigns a tangent vector $\bm{v}\in\tangentspace{\bm{x}}$ to each point $\bm{x} \in \manifold$. A Riemannian manifold is a smooth manifold equipped with a symmetric, non-degenerate, positive-definite bilinear 2-form $g_{\bm{x}}: \tangentspace{\bm{x}} \times \tangentspace{\bm{x}} \to \mathbb{R}$ which changes smoothly with $\bm{x} \in \manifold$, called the Riemannian metric. The connection operator $\nabla:\mathfrak{X}(\manifold) \times \mathfrak{X}(\manifold) \to \mathfrak{X}(\manifold)$ allows the differentiation of vector fields along other vector fields. It plays a key role in Riemannian geometry as it allows the definition of Riemannian operations such as the Riemannian Hessian and the exponential map. While a smooth manifold admits many different connections, the Levi-Civita one is the most widely adopted as the unique metric-compatible and torsion-free connection. 

The Riemannian equivalents of gradient and Hessian are fundamental for Riemannian optimization. The Riemannian gradient of a function $f:\manifold\to\euclideanspace$ is defined as the (unique) vector field such that, for all $X \in \mathfrak{X}(\manifold)$, $\langle \grad_{\bm{x}} f, X \rangle_{\bm{x}}= \operatorname{d}_{\bm{x}} f[X]$ where $\operatorname{d}$ is the exterior differential. 
The Riemannian Hessian along $X \in \mathfrak{X}(\manifold)$ is defined via the connection as
$\mathcal{H}_{\bm{x}}f[X] = \nabla_X \grad_{\bm{x}} f$.
In Riemannian optimization algorithms, the search direction is obtained from the Riemannian gradient and Hessian and lies on the tangent space. The next iterate is then obtained by projecting the search direction onto the manifold via the exponential map, which defines a local diffeomorphism $\expmapblank{\bm{x}}:\mathcal{V}_{\bm{0}} \to \mathcal{U}_{\bm{x}}$ between sufficiently small neighborhoods $\mathcal{U}_{\bm{x}} \subseteq \manifold$ and $\mathcal{V}_{\bm{0}} \subset \tangentspace{\bm{x}}$. 

\subsection{Information geometry of the probability simplex}
\label{subsec: IGprobsimplex}

The $d$-dimensional probability simplex is the smooth manifold defined as ${\simplex{d} =\{\bm{x} \in \mathbb{R}^{d+1}|x_i \geq 0, \sum_{i=1}^{d+1}x_i=1\}}$, as illustrated in Fig~\ref{fig:teaser}-\emph{left}. The tangent space at each $\bm{\mu}\!\in\!\simplex{d}$ is defined as ${\tangentsimplex{d}{\bm{x}} = \{\bm{v} \in \mathbb{R}^{d+1}| \bm{x}^\trsp \bm{v}=0\}}$. The probability simplex can be endowed with the Fisher-Rao metric $g_{\bm{x}}(\bm{u}, \bm{v}) = \sum_i x_i u_i v_i$. 

In information geometry, it is common to work with the conjugate connection structure, which naturally arises from the geometric description of statistical models and allows the construction of a family of connections. Formally, the second order geometry is induced by conjugate connections $(\nabla, \nabla^*)$ which are related as 
\begin{equation}
    Xg(Y,Z) = g(\nabla_X Y,Z) + g(Y, \nabla^*_X Z)\mbox{,}
\end{equation}
for vector fields $X, Y, Z \in \mathfrak{X}(\Delta^d)$. For the probability simplex, the conjugate connections, also called mixture and exponential connections, are 
\begin{align}
    \nabla_XY(\bm{x}) &= (\bm{x}, D^2_XY(\bm{x}) + X(\bm{x})Y(\bm{x}))\mbox{,} \label{eq:exponential-connection}\\
    \nabla^*_XY(\bm{x}) &= \left(\bm{x}, D^2_XY(\bm{x}) +  \bm{x}^T \left(X(\bm{x})Y(\bm{x})\right)\right)\mbox{,} \label{eq:mixture-connection}
\end{align}
where $D^2$ denotes the Euclidean differentiation of vector fields. 
The $\alpha$-connection is the one-parameter family of connections defined via the convex combinations of~\eqref{eq:exponential-connection}-\eqref{eq:mixture-connection} as\looseness-1
\begin{equation}
\label{eq:alpha-connection}
    \nabla^{(\alpha)}_XY(\bm{x}) = \frac{1+\alpha}{2} \nabla_XY(\bm{x}) + \frac{1-\alpha}{2}\nabla^*_XY(\bm{x}),
\end{equation}
with $\alpha \in [-1,1]$.
Note that we recover the mixture and exponential connections for $\alpha=1$ and $\alpha=-1$.
Moreover, the Levi-Civita connection is obtained as $\nabla^{(\text{LC})}=\frac{\nabla + \nabla^*}{2}$ for $\alpha=0$.
The choice of $\alpha$ has a significant impact on optimization, as it determines the related Riemannian operations. A notable example is the exponential map, whose domain is a bounded Euclidean subset with border for all $\alpha\neq-1$. Instead, when $\alpha=-1$, the domain is the full Euclidean space, i.e., $\mathcal{V}_{\bm{0}}\cong \euclideanspace^d$, except for $\bm{x}$ on the simplex boundary, where it is some half-space $\{\euclideanspace^k_+ \times \euclideanspace^{d-k}:k<d\}$. In other words, all tangent vectors at a point $\bm{x}$ in the interior of the simplex are within the domain of the exponential map and the border of the simplex can be reached by mapping tangent vectors at the limit to infinity. 
A more detailed treatment can be found in standard information geometry textbooks~\citep{amari2016information, ay2017information} or in~\citep{Pistone2019Information}, see also App.~\ref{appendix:IGSimplex}.

\subsection{Bayesian optimization}
\label{subsec:BO}

Bayesian optimization (BO) is a class of sequential search algorithms aiming to optimize an unknown objective function $F$ defined over a domain $\mathcal{X}$. The function $F$ has no simple closed form, but can be observed point-wise and up-to-noise at query points $\bm{x}\!\in\!\mathcal{X}$. 
BO switches the optimization problem to a surrogate a random field $f$ approximating $F$. Optimization is then carried by iteratively choosing the query point based on its utility according to the surrogate.  

The approximating random field is commonly chosen as a Gaussian process $f \sim \mathcal{GP}(m, k)$, uniquely specified by its mean $m\!:\!\mathcal{X} \!\to\! \euclideanspace$ and positive-definite kernel, or covariance function, $k\!:\!\mathcal{X} \!\times\! \mathcal{X} \!\to\! \euclideanspace$~\citep{williams1995gaussian}. This latter is the fundamental element of Gaussian processes as it defines similarities between points of the domain, thereby encoding our prior about the unknown objective function $F$.
At each iteration $N$, the Gaussian process is conditioned on the collected dataset $D = \{(\bm{x}_i, y_i)\}_{i=1}^N$, where $y_i \in \mathbb{R}$ is the realization of the objective function at $\bm{x}_i$, possibly corrupted by Gaussian-centered noise with variance $\sigma^2$. The posterior process is obtained by updating moments via Bayes rule, resulting in a Gaussian process with mean and covariance functions
\begin{align}
    \label{eq:gp_mean}
    m_D(\bm{x}) &= m(\bm{x}) + k(\bm{x}, D) \bm{\Sigma}^{-1} (\bm{y} - m(D)), \\ 
    \label{eq:gp_cov}
    k_D(\bm{x}, \bm{x}') &= k(\bm{x}, \bm{x}') - k(\bm{x}, D) \bm{\Sigma}^{-1} k(D, \bm{x}'),
\end{align}
with $\Sigma = K(D, D) + \sigma^2 I$, $\bm{y}=(y_1, \dots, y_N)^\trsp$ and where $k(\bm{x}, D)$ is the row vector with components $k_{i} = k(\bm{x}, \bm{x}_i)$.

The next query point is then selected by optimizing an acquisition function constructed from the posterior process. Intuitively, the acquisition function quantifies the utility of evaluating the objective function at a given point, resolving the \emph{explore-exploit} tradeoff~\citep{shahriari2015:surveyBO}. The next point to evaluate is then chosen as the maximizer of the acquisition function, i.e., $\bm{x}_{N+1} = \operatorname{arg}\max_{\bm{x} \in \mathcal{X}} A(f(\bm{x}|D))$.
In contrast to the original function $F$, the acquisition function is differentiable and cheap to evaluate. In this paper, we employ two common acquisition functions, namely Expected Improvement (EI)~\citep{Mockus1978:EI} and Lower Confidence Bound (LCB)~\citep{Auer2022:UCB}. 
Finally, the target function $y_{N+1}=F(\bm{x}_{N+1})$ is evaluated at $\bm{x}_{N+1}$ and added to the dataset $D$. The procedure is iterated by updating the posterior process~\eqref{eq:gp_mean}-\eqref{eq:gp_cov} and maximizing the acquisition function until a certain criterion is met.
It is possible to prove~\citep{garnett_bayesopt} that, under certain assumptions on the parameter space, BO converges to the optimum. 

\subsection{Geometry-aware BO}
\label{subsec:GaBO}

Geometry-aware Bayesian optimization (GaBO) addresses the case where the search space $\mathcal{X}$ is a Riemannian manifold instead of Euclidean space, as in Sec.~\ref{subsec:BO}.
This is achieved by (1) defining positive-definite kernels measuring the similarity of points on the manifold, and (2) optimizing the acquisition function on the manifold.

\paragraph{Kernels on manifolds.}
\label{parag:pdkernels}
Defining a valid kernel on a Riemannian manifold is a non-trivial task as naive approaches that replace the Euclidean distance with the Riemannian one in isotropic kernels fail to capture the manifold's geometric structure. Such kernels are not guaranteed to be positive definite, and do not necessarily preserve the stationarity or sample path regularity properties of their Euclidean counterparts~\citep{feragen2015geodesic, jaquier20a, jaquier20b}. 
\citet{borovitskiy2020matern} developed solutions for compact manifolds by characterizing Matérn kernels by the spectral decomposition of the Laplace-Beltrami operator, leading to the general form
\begin{equation}
    \label{eq:compact-kernel}
    k_\nu(x, x') = \frac{\sigma^2}{C_\nu} \sum_{n=0}^{\infty} \Phi(\lambda_n) \phi_n(x)\phi_n(x'),
\end{equation}
where $\phi_n$ are the Laplace-Beltrami eigenfunctions, $\Phi$ is an operator defining the kernel by acting on the respective eigenvalues $\lambda_n$, $\sigma$ is an hyperparameter, and $C_\nu$ is a normalizing constant. 
Extensions to Lie groups and homogeneous spaces are provided by~\citep{azangulov2024stationary, azangulov2024stationary2}.
This class of Riemannian kernels was shown to lead to enhanced performances for BO on manifolds~\citep{jaquier22a}. 

\paragraph{Acquisition function optimization on manifolds.}
\label{parag:acqfunmanifolds}
In the geometric setting, the acquisition function is a real-valued function expressed in terms of Gaussian process posterior moments over the Riemannian domain $\mathcal{X}$. \citet{jaquier20a} proposed to leverage Riemannian optimization algorithms~\citep{absil2008optimization, Boumal22:RiemannOpt} to optimize the acquisition function while accounting for the manifold's geometry. In this setting, the Riemannian gradient and Hessian functions are used to determine a search direction on the tangent space $\mathcal{T}_{\bm{x}}\mathcal{X}$, which is then mapped from the tangent space back onto the manifold using the exponential map. Similarly to the Euclidean case, the search direction is determined by the optimization algorithm at hand.

\section{$\alpha$-simplex Geometry-aware Bayesian Optimization}
\label{sec:alpha-GaBO}
In this section, we introduce $\alpha$-simplex Geometry-aware Bayesian Optimization ($\alpha$-GaBO), a one-parameter family of BO algorithms that handle the probability simplex as optimization domain, i.e., $\mathcal{X}=\simplex{d}$. While the family is theoretically well-defined for $\alpha \in [-1, 1]$, we subsequently focus on two parameter values $\alpha=\{-1, 0\}$, leading to closed-form expressions for the Riemannian operations.
Similarly to standard GaBO, the key elements of $\alpha$-GaBO are: (1) Gaussian processes equipped with kernels that properly capture similarities between points on the simplex, and (2) optimization methods for maximizing the acquisition function over this domain.
Unlike the manifolds considered in~\citep{jaquier20a,jaquier22a}, the probability simplex is a Riemannian manifold with boundary, thus introducing two specific challenges.
First, the literature does not provide a general construction for positive-definite kernels on manifolds with boundary.
Second, considering the boundary is essential in practice, since optimal solutions may lie on vertices or $k$-faces of the simplex. In particular, vertices correspond to degenerate distributions concentrating all mass on a single component, while $k$-faces correspond to mixtures of $k<d$ components.
We address the first challenge by exploiting the isometry between the simplex and a subset of the sphere (see Fig.~\ref{fig:teaser}) to construct valid kernels.
To handle the second challenge, we introduce a family of information-geometric optimizers based on the Riemannian $\alpha$-connection~\eqref{eq:alpha-connection}, where the parameter $\alpha$ allows us to incorporate prior knowledge about the information geometry of the simplex.

\subsection{Kernels on the probability simplex}
\label{subsec:simplex-kernels}
The kernel~\eqref{eq:compact-kernel} is valid only for compact Riemannian manifolds without boundary. Kernels on manifolds with boundary are not treated in literature and existing results do not trivially extend to this setting. For the probability simplex, we propose to leverage an isometry with a compact subset of the sphere on which the construction of Matérn kernels is well explored~\citep{borovitskiy2020matern}. Formally, we denote the positive orthant of the hypersphere of radius $r$ as
\begin{equation}
    \sphere{d}_{\geq0}(r)=\{\bm{x}\in \mathbb{R}^{d+1}| x_i\geq0,\|\bm{x}\|_2=r \}\mbox{.}
\end{equation}
The sphere map, defined as~\citep{amari2016information, ay2017information}
\begin{equation}
\label{eq:sphere_map}
    \varphi:\Delta^d \to \sphere{d}_{\geq0}(2): \bm{x} \mapsto 2\sqrt{\bm{x}}\mbox{,}
\end{equation}
where the square root is taken element-wise, is a diffeomorphism between the probability simplex and the positive orthant $\sphere{d}_{\geq0}(2)$. When endowing the probability simplex with the Fisher-Rao metric and the sphere with the usual restriction metric, the sphere map is also an isometry. In other words, it identifies the Riemannian geometry of the simplex $\simplex{d}$ to that of a hypersphere $\sphere{d}(2)$, thereby preserving distances and other geometric operations. Note that, in practice, it is more convenient to work with the standard radius-$1$ sphere $\sphere{d}$. This is achieved by scaling the map as $\frac{1}{2}\varphi$ and appropriately rescaling the metric to maintain the isometry. More details on the sphere map are in App.~\ref{appendix:IGSpheremap}. 

We utilize the sphere map to construct kernels on the probability simplex from kernels on the hypersphere. Matérn kernels on $\sphere{d}$ have the form~\citep{borovitskiy2020matern}
\begin{equation}
\label{eq:kernel-sphere}
    k_{\nu, \kappa, \sigma}^{\sphere{d}}(\bm{s}, \bm{s}') = \frac{\sigma^2}{C_\nu}\sum_{n=0}^\infty c_{n,d} \rho_{\nu}(n)\mathcal{C}_n^{\frac{d-1}{2}} \big( \cos\left(d_g(\bm{s}, \bm{s}')\right)\big)\mbox{,}
\end{equation}
where $\bm{s}, \bm{s}' \!\in\! \sphere{d}$, $d_g$ is the Riemannian distance, $\mathcal{C}_n^{(\cdot)}$ are Gegenbauer polynomials, $\rho_\nu(n)$ are the generalized Fourier coefficients of the spectral measure, $c_{n,d}$ are explicit constants, and $C_\nu$ is a normalization constant ensuring $k_{\nu, \kappa, \sigma}^{\sphere{d}}(\bm{s}, \bm{s}')=\sigma^2$. Explicit formulas for $\rho_\nu(n)$ and $c_{n,d}$ are in~\citep{borovitskiy2020matern}. 
As the spectral series of the kernel converges uniformly, the approximation precision is chosen by truncating the series at a certain $N$. 
We define Matérn kernels on the $d$-dimensional probability simplex as the pullback of the sphere kernel~\eqref{eq:kernel-sphere} via the sphere map as
\begin{equation}
\label{eq:kernel-simplex}
    k^{\simplex{d}}(\bm{x}, \bm{x}') = \frac{1}{2}\varphi^* \tilde{k}^{\sphere{d}}(\bm{x}, \bm{x}')=\tilde{k}^{\sphere{d}}\big(\frac{1}{2}\varphi(\bm{x}), \frac{1}{2}\varphi(\bm{x}')\big)\mbox{,}
\end{equation}
where $\bm{x}, \bm{x}' \in \Delta^d$ and $\tilde{k}_{\nu, \kappa, \sigma}^{\sphere{d}}$ denotes the restriction of~\eqref{eq:kernel-sphere} to $\sphere{d}_{\geq0}$. Note that the smoothness of the isometry guarantees the differentiability of $k_{\nu, \kappa, \sigma}^{\simplex{d}}$ up to the order of $k_{\nu, \kappa, \sigma}^{\sphere{d}}$. 

\begin{algorithm}[t]
\caption{$\alpha_{\text{-}1}$-GaBO}
\label{alg:mone_alphagabo}
\begin{algorithmic}[1]

\Require Objective function $F:\simplex{d}\!\to\!\euclideanspace$, initial dataset $D_0=\{\bm{x}_m,y_m\}_{m=1}^M$ with $\bm{x}_m\in\simplex{d}$,  budget $B$
\Ensure Recommendation $\bm{x}^*\in\simplex{d}$

\State Place the prior: 
$f \sim \mathcal{GP}(m, \textcolor{gabored}{k_{\theta}^{\simplex{d}}})$ with~\eqref{eq:kernel-simplex}

\State Initialize the dataset on $\textcolor{gabored}{\simplex{d}}$: $D \gets D_0$

\For{$N = 1,\ldots,B$}

    \State Update the posterior: $f \mid D \sim \mathcal{GP}(m_D, \textcolor{gabored}{k_{\theta,D}^{\simplex{d}}})$

    \State Select the next query:
    \Statex \hspace*{\algorithmicindent} $\bm{x}_{N} =
            \operatorname{arg}\max_{\textcolor{gabored}{\bm{x} \in \simplex{d}}} A(f(\bm{x}|D))$ 

    \State Evaluate the objective: $y_N \gets F(\bm{x}_{N})$
    \State Update the dataset: $D \gets D \cup \{(\bm{x}_N,y_N)\}$

\EndFor

\State \Return $(\bm{x}^*,y^*) = \operatorname{arg}\max_{(\bm{x},y)\in D} y$

\end{algorithmic}
\end{algorithm}

\subsection{Acquisition function optimization}
\label{subsec:simplex-acqfct}

After conditioning the Gaussian process with kernel~\eqref{eq:kernel-simplex} modeling the unknown function $F$, the next query point $\bm{x}_{N+1}\in\simplex{d}$ is obtained by maximizing the acquisition function $A$. Similar to GaBO~\citep{jaquier20a,jaquier22a}, we leverage Riemannian optimization algorithms~\citep{absil2008optimization,Boumal22:RiemannOpt} to account for the geometry of the probability simplex. 
We propose a one-parameter family of acquisition function optimizers based on the $\alpha$-connection~\eqref{eq:alpha-connection}. Specifically, the search direction ${\bm{\eta}_k\in\tangentsimplex{d}{\bm{x}_k}}$ and iterate $\bm{x}_k\in\simplex{d}$ are iteratively updated as
\begin{equation}
    \bm{\eta}_k = h(\bm{x}_k, \grad_{\bm{x}_k} A, \mathcal{H}_{\bm{x}_k}^\alpha A) \quad \text{and} \quad \bm{x}_{k+1} = \expmapblank{\bm{x}_{k}}^\alpha(\eta_k),
\end{equation}
where the function $h$ is defined based on the choice of optimization algorithm, e.g., for a simple gradient descent $h = - \gamma \; \grad_{\bm{x}_k} A$, and the Riemannian operations are determined based on the choice of parameter $\alpha$.
The choice of this parameter reflects prior knowledge according to the connection between the simplex and the reference statistical model~\citep{amari2016information, ay2017information}. 
We focus on two parameter values $\alpha = \{-1, 0\}$ leading to closed-form expressions for the Riemannian operations, see App.~\ref{appendix:IGSimplex}. The resulting algorithms, hereinafter referred to as $\alpha_{\text{-}1}$- and $\alpha_{0}$-GaBO, are summarized in Algorithms~\ref{alg:mone_alphagabo} and~\ref{alg:zero_alphagabo}.

\begin{algorithm}[t]
\caption{$\alpha_{0}$-GaBO}
\label{alg:zero_alphagabo}
\begin{algorithmic}[1]

\Require Objective function $F:\simplex{d}\!\to\!\euclideanspace$, initial dataset $D_0=\{\bm{x}_m,y_m\}_{m=1}^M$ with $\bm{x}_m\in\simplex{d}$, budget $B$
\Ensure Recommendation $\bm{x}^*\in\simplex{d}$ 

\State Map $D_0$ to the sphere: $\textcolor{gaboyellowdark}{D_0^{\sphere{d}}}=\{\textcolor{gaboyellowdark}{\frac{1}{2}\varphi}(\bm{x}_m),y_m\}_{m=1}^M$
\State Place the prior: 
$f \sim \mathcal{GP}(m, \textcolor{gaboyellowdark}{k_{\theta}^{\sphere{d}}})$ with~\eqref{eq:kernel-sphere}

\State Initialize the dataset on $\textcolor{gaboyellowdark}{\sphere{d}_{\geq0}}$: $D \gets \textcolor{gaboyellowdark}{D_0^{\sphere{d}}}$

\For{$N = 1,\ldots,B$}

    \State Update the posterior: $f \mid D \sim \mathcal{GP}(m_D, \textcolor{gaboyellowdark}{k_{\theta,D}^{\sphere{d}}})$

    \State Select the next query:
    \Statex \hspace*{\algorithmicindent}$\bm{s}_{N} =
            \operatorname{arg}\max_{\textcolor{gaboyellowdark}{\bm{s} \in \sphere{d}_{\geq0}}} A(f(\bm{s}|D))$ 

    \State Evaluate the objective: $y_N \gets F(\textcolor{gaboyellowdark}{2\varphi^{-1}}(\bm{s}_{N}))$
    \State Update the dataset: $D \gets D \cup \{(\bm{s}_N,y_N)\}$

\EndFor

\State \Return $(\bm{x}^*,y^*) = \operatorname{arg}\max_{(\textcolor{gaboyellowdark}{2\varphi^{-1}}(\bm{s}),y)\in D} y$
\end{algorithmic}
\end{algorithm}

By recovering the exponential connection~\eqref{eq:exponential-connection}, $\alpha_{\text{-}1}$-GaBO allows us to perform unconstrained Riemannian optimization on $\simplex{d}$. The corresponding exponential map is given as $\expmap{\bm{x}}{\bm{\eta}} = \frac{\exp(\bm{\eta})}{\exp(\bm{\eta})^\trsp \bm{x}} \odot \bm{x}$ with the exponential taken component wise and $\odot$ denotes the Hadamard product. 
It ensures that, for each point $\bm{x}$ in the interior of $\simplex{d}$, the domain of the exponential map is the entire tangent space $\mathcal{T}_{\bm{x}}\simplex{d} \cong \euclideanspace^d$. In other words, no tangent vector $\bm{\eta}\in\tangentsimplex{d}{\bm{x}}$ is mapped outside of $\simplex{d}$ and the border is reached at $\|\bm{\eta}\|\to\infty$. However, this implies that optimization methods cannot reach the border, additionally inducing numerically instabilities close to it. Therefore, $\alpha_{\text{-}1}$-GaBO is not suited for problems where the optimum may lie on the border of the simplex.

$\alpha_{0}$-GaBO recovers the Levi-Civita connection, equally balancing exponential and mixture geometry. The exponential map is ${\expmap{\bm{x}}{\bm{\eta}}=\big(\sqrt{\bm{x}} \cos (\frac{\| \bm{\eta}\|_F}{2}) + \frac{\sqrt{\bm{x}} \odot \bm{\eta}}{\|\bm{\eta}\|_F} \sin( \frac{\| \bm{\eta}\|_F}{2})\big)^2}$ where $\|\bm{\eta}\|_F\!=\!\sqrt{\bm{x}^T\bm{\eta}^2}$ and exponentiation and square root are taken component-wise. It allows tangent vectors to be mapped onto the border of $\simplex{d}$ at the expense of a constrained exponential map domain, as generic tangent vectors might be mapped outside the simplex.
As $\Delta^d$ is isometric to $\sphere{d}_{\geq0}$ and the Levi-Civita connection is the only metric-compatible connection, the geometric structure of the probability simplex is, for $\alpha=0$,  equivalent to the standard hypersphere geometry. Therefore, each iteration of $\alpha_{0}$-GaBO corresponds to conditioning a sphere Gaussian process with kernel~\eqref{eq:kernel-sphere}, maximizing the acquisition function via Riemannian optimization algorithms constrained to the sphere positive orthant $\sphere{d}_{\geq0}$~\citep{liu2020:RiemOpt-constrained}, and mapping the obtained query point $\bm{x}_{N+1}$ onto $\simplex{d}$ via the inverse sphere map, see also Algo.~\ref{alg:zero_alphagabo}. 

\textbf{Choice of $\alpha$.} When the simplex strictly represents probability distributions, its Riemannian structure is naturally characterized in terms of a $\alpha$-divergence, which determines the appropriate value of $\alpha$~\citep{amari2016information, ay2017information}.
In other cases, the most suitable $\alpha$ depends on the prior knowledge about the location of the optimum. When the optimum may lie on the border of $\simplex{d}$, we recommend using $\alpha_{0}$-GaBO. Conversely, when the optimum is expected to lie in the interior, we suggest to use $\alpha_{\text{-}1}$-GaBO as it enables simpler unconstrained Riemannian optimization over the entire simplex. Note that our experiments show that both $\alpha_{0}$- and $\alpha_{\text{-}1}$-GaBO perform similarly when the optimum lies in the interior, showcasing that constraining the optimization to the sphere positive orthant is not an issue in practice.

\begin{figure*}[t]
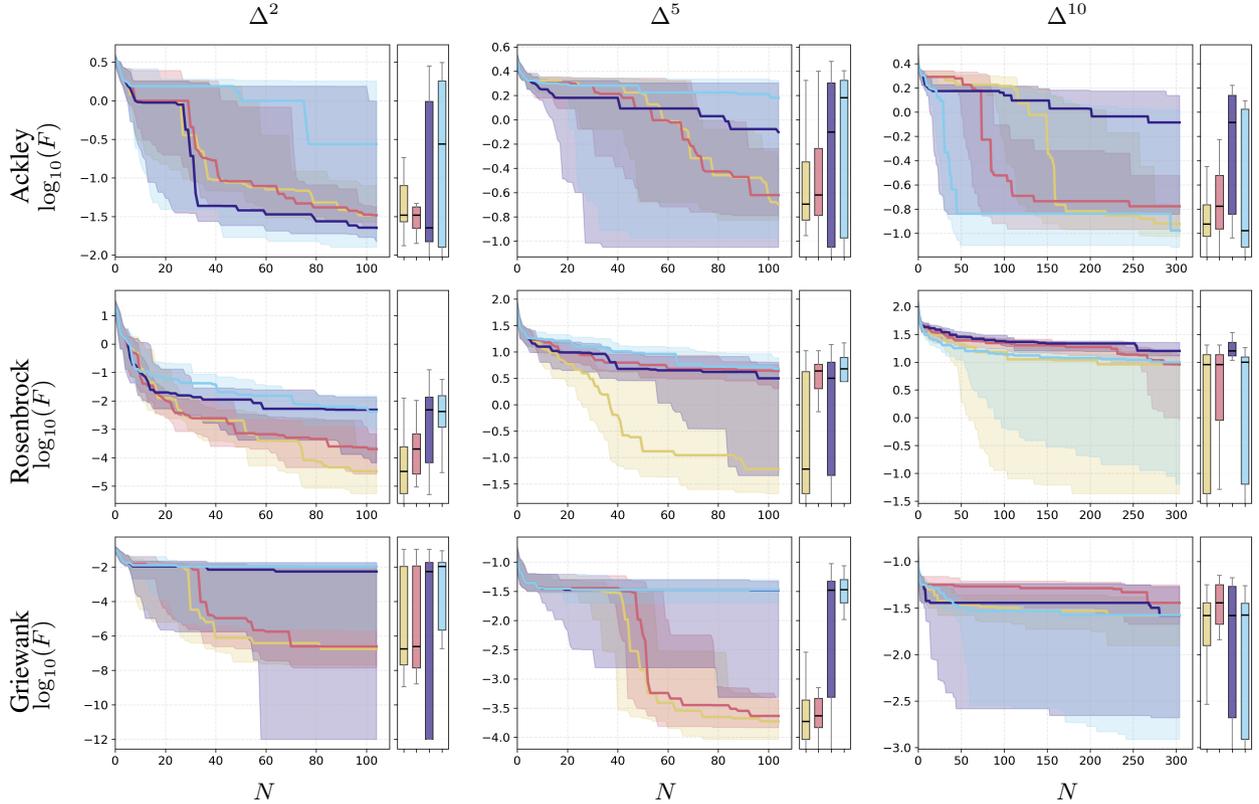

    \centering
    
    \setlength{\tabcolsep}{2pt}
    \renewcommand{\arraystretch}{1.2} 
    
    \begin{tabular}{m{0.7cm} c c c} 
        &\footnotesize  $\Delta^2$ &\footnotesize  $\Delta^5$ &\footnotesize  $\Delta^{10}$ \\[0.0em] 
        \begin{tabular}{c @{\hspace{1pt}} c}  
            \rotatebox[origin=c]{90}{Ackley} & \rotatebox[origin=c]{90}{\footnotesize $\log_{10}(F)$}
        \end{tabular} &
        \adjustimage{valign=c,width=0.3\textwidth}{Ackley_dim2_formal_with_alpha.png} &
        \adjustimage{valign=c,width=0.3\textwidth}{Ackley_dim5_formal_with_alpha.png} &
        \adjustimage{valign=c,width=0.3\textwidth}{Ackley_dim10_formal_with_alpha.png} \\[1em] 
        \begin{tabular}{c @{\hspace{1pt}} c} 
            \rotatebox[origin=c]{90}{Rosenbrock} & \rotatebox[origin=c]{90}{\footnotesize $\log_{10}(F)$}
        \end{tabular} &
        \adjustimage{valign=c,width=0.3\textwidth}{Rosenbrock_dim2_formal_with_alpha.png} &
        \adjustimage{valign=c,width=0.3\textwidth}{Rosenbrock_dim5_formal_with_alpha.png} &
        \adjustimage{valign=c,width=0.3\textwidth}{Rosenbrock_dim10_formal_with_alpha.png} \\[1em] 
        \begin{tabular}{c @{\hspace{1pt}} c} 
            \rotatebox[origin=c]{90}{Griewank} & \rotatebox[origin=c]{90}{\footnotesize $\log_{10}(F)$}
        \end{tabular} &
        \adjustimage{valign=c,width=0.3\textwidth}{Griewank_dim2_formal_with_alpha.png} &
        \adjustimage{valign=c,width=0.3\textwidth}{Griewank_dim5_formal_with_alpha.png} &
        \adjustimage{valign=c,width=0.3\textwidth}{Griewank_dim10_formal_with_alpha.png} \\ \vspace{-4pt} 
        & \footnotesize $N$ & \footnotesize $N$ & \footnotesize  $N$ \\
    \end{tabular}
    \caption{Logarithm of the regret (median and quartiles) and distribution over the final recommendation for $\alpha_0$-GaBO (\aogaboline), 
\mbox{$\alpha_{\text{-}1}$-GaBO} (\amogaboline), $\sphere{d}$-Eucl. BO (\spheuboline), and BORIS (\borisline) on benchmark functions.
}
\label{fig:benchmarks}
\end{figure*}

\begin{figure*}[t]
    \centering
    \setlength{\tabcolsep}{2pt} 
    \renewcommand{\arraystretch}{1.2} 
    \begin{tabular}{m{0.2cm} c c c}
        \raisebox{6.5ex}{\rotatebox[origin=c]{90}{\footnotesize $F$\strut}} &
        \begin{minipage}[c]{0.3\textwidth}
            \centering
            \includegraphics[width=\linewidth]{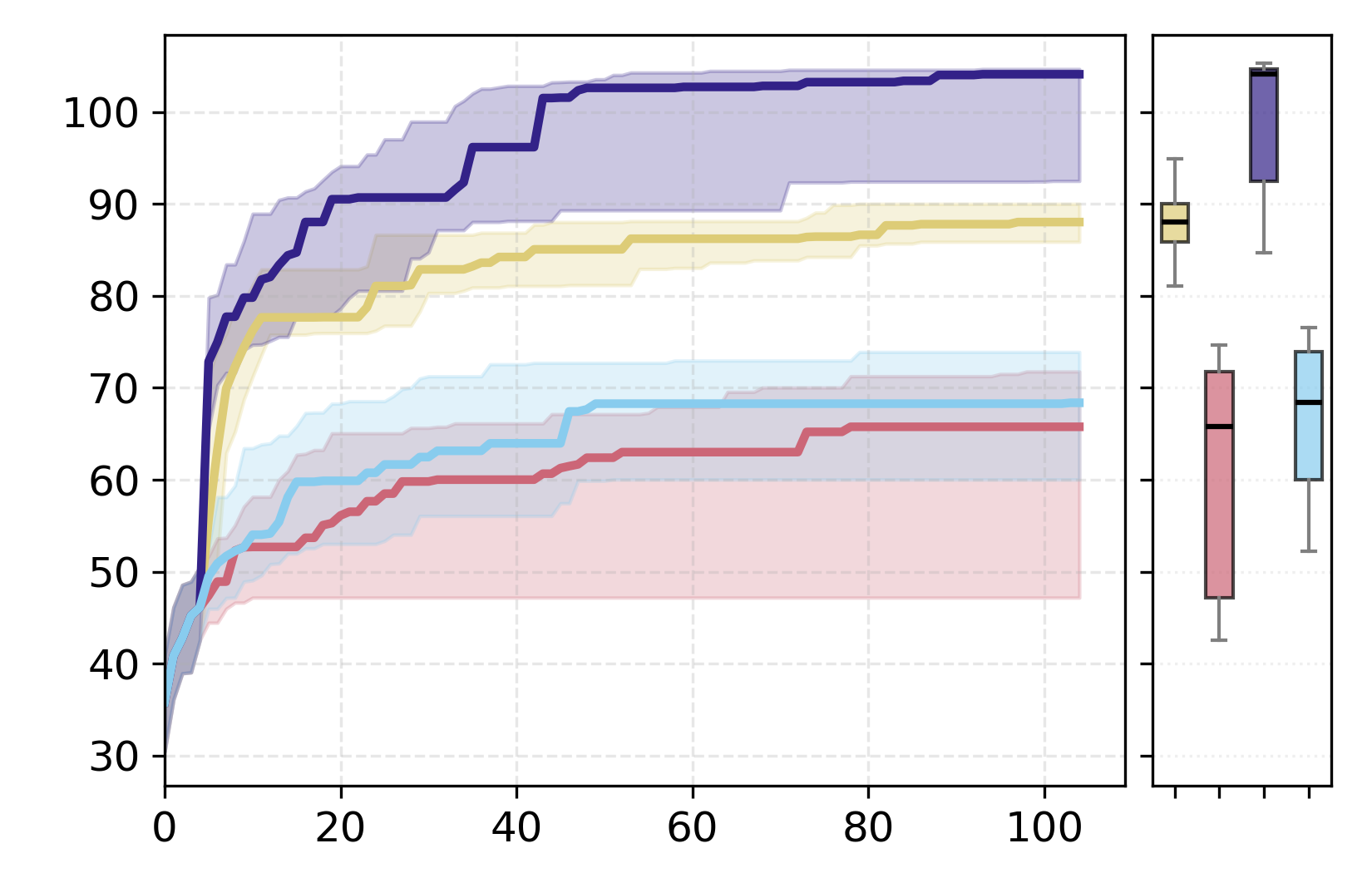}\par
            \vspace{-4pt} 
             \footnotesize $N$\par
            \subcaption{Concrete, $\Delta^6 \times [0,1]$.}
            \label{subfig4:concrete}
        \end{minipage} &
        \begin{minipage}[c]{0.3\textwidth}
            \centering
            \includegraphics[width=\linewidth]{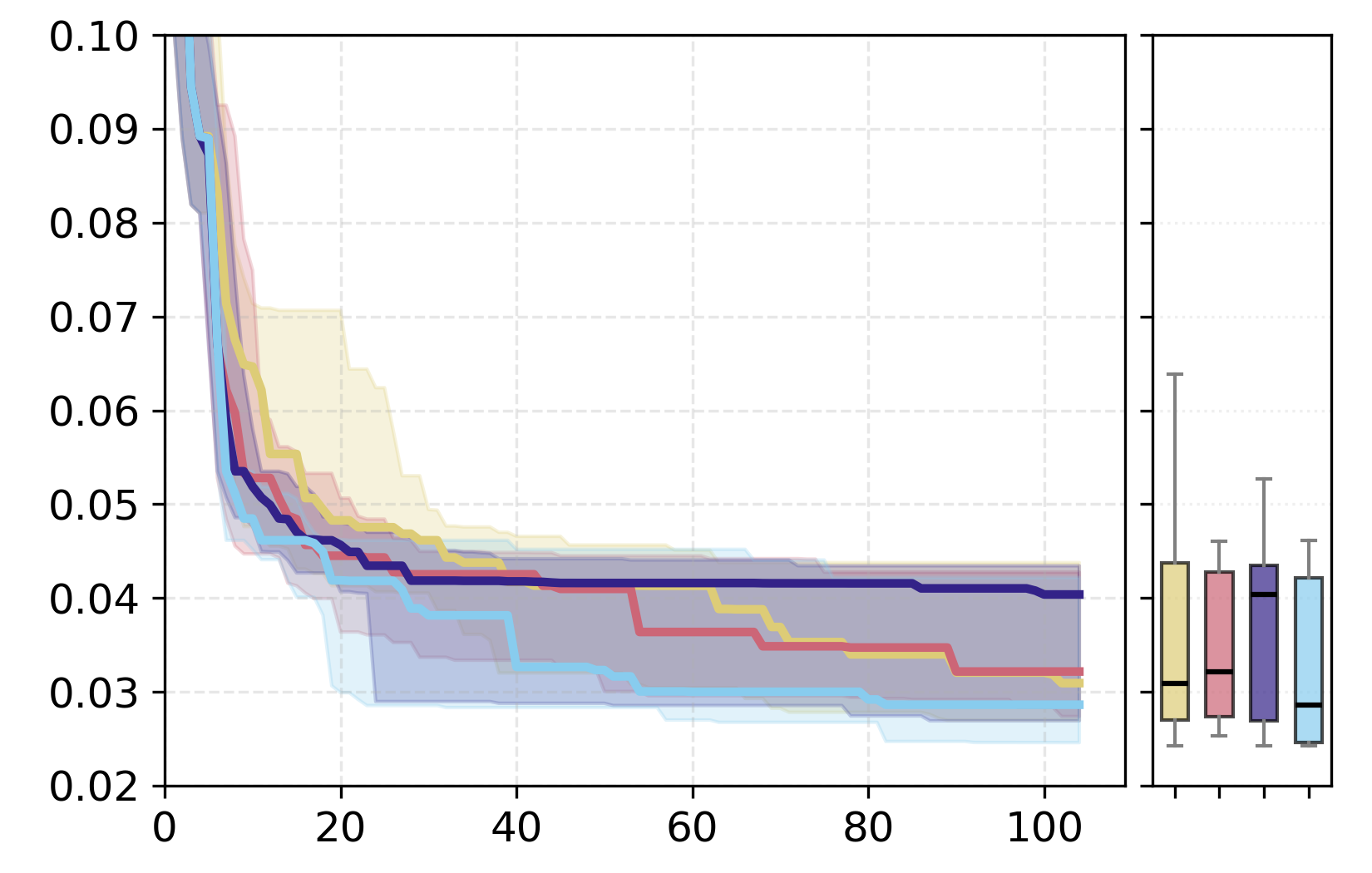}\par
            \vspace{-4pt} 
             \footnotesize $N$\par
            \subcaption{Photovoltaic WF3, $\Delta^3$.}
            \label{subfig4:wf3}
        \end{minipage} &
        \begin{minipage}[c]{0.3\textwidth}
            \centering
            \includegraphics[width=\linewidth]{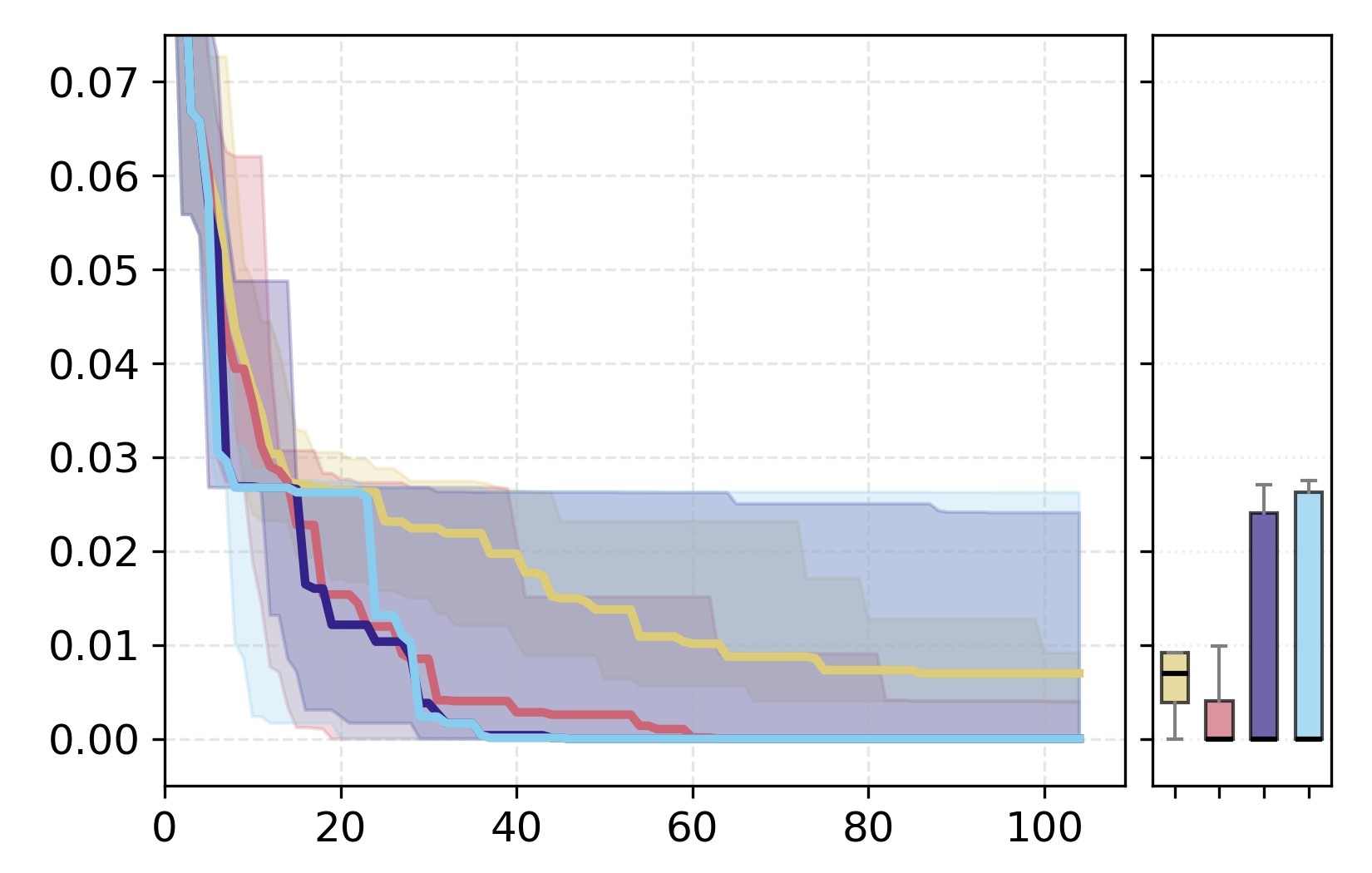}\par
            \vspace{-4pt} 
             \footnotesize $N$\par
            \subcaption{Photovoltaic PCE10, $\Delta^3$.}
            \label{subfig4:pce10}
        \end{minipage} \\
    \end{tabular}
    \caption{Regret (median and quartiles) and distribution over the final recommendation for $\alpha_0$-GaBO (\aogaboline), \mbox{$\alpha_{\text{-}1}$-GaBO} (\amogaboline), $\sphere{d}$-Eucl. BO (\spheuboline), and BORIS (\borisline) on optimal mixture datsets. Note that the objective function is maximized for concrete.}
    \label{fig:realworld_all}
\end{figure*}

\section{Experiments}
\label{sec: experiments}

We test \mbox{$\alpha_{\text{-}1}$-GaBO} and $\alpha_0$-GaBO on several benchmark test functions and on three real-world applications. 
We use the Riemannian squared exponential kernel motivated by the fact for the Riemannian case the series converges faster than Matérn ones~\citep{borovitskiy2020matern}, thus limiting the truncation error. An exception is made for Olympus, where we found it beneficial to use $\nu\!=\!\frac{5}{2}$ Riemannian Matérn kernel. We use EI~\citep{Mockus1978:EI} as acquisition function for all experiments, except for Olympus where we observe better performances using LCB~\citep{Auer2022:UCB}. The acquisition function is optimized using (constrained) Riemannian trust regions~\citep{absil07:trust-region,jaquier20b}. The practical implementation can be viewed as an extension of the GaBOTorch library~\citep{jaquier20a, jaquier22a}.\looseness-1

We compare $\alpha$-GaBO against standard constrained Euclidean alternatives~\citep{jaquier20a}, where a standard Euclidean BO procedure with Euclidean kernel is defined in the ambient space and constrained on the manifold of interest to optimize the acquisition function. Consistently with the two considered $\alpha$-GaBO models, we test constrained Euclidean BO on the probability simplex and on the sphere coupled with the sphere map~\eqref{eq:sphere_map}. Note that, in light of~\citep{candelieriBO2}, constrained Euclidean BO on the probability simplex is equivalent to BORIS~\citep{candelieriBO}. For all the experiments, we initialize the BO algorithm with $5$ random sample and report results across $25$ different seeds.

\subsection{Benchmark Functions}
\label{subsec:benchmarks}

We study the performance of $\alpha$-GaBO on the Ackley, Rosenbrock and Griewank benchmark functions projected on the probability simplex. Following the technique by~\citep{jaquier20a}, we represent the test function on a tangent space $\tangentsimplex{d}{\bm{x}}$ at a point $\bm{x} \in \simplex{d}$ and project it onto $\simplex{d}$ via the inverse of the exponential map, called logarithmic map $\logmapblank{\bm{x}}: \simplex{d} \to \tangentsimplex{d}{\bm{x}}$. 
We pick the base point $\bm{x}$ as the center of the probability simplex. The optimum of the Ackley, Rosenbrock and Griewank benchmark function is located at this point by definition. In the benchmark experiments, we choose an optimum in the interior of the simplex to ensure a fair comparison between all models, as placing it on the border would have penalized $\alpha_{\text{-}1}$-GaBO. For all benchmarks, we run the experiments for dimensions $d \in \{2, 5, 10\}$. 

Fig.~\ref{fig:benchmarks} shows the evolution of the median and the distribution of the logarithm of the simple regret of the final recommendation. We observe that the $\alpha$-GaBO models generally match or outperform their constrained Euclidean counterparts by converging more data-efficiently to lower function values and with lower variance over the final recommendation. This is especially evident for dimensions $d=\{2,5\}$. Despite attaining lower values for Ackley on $\simplex{2}$, the Euclidean BO on the sphere is characterized by a much higher variance than $\alpha$-GaBO, indicating less consistent performance. The differences across models are less pronounced when the dimensionality increases. 
Finally, $\alpha_0$-GaBO and \mbox{$\alpha_{\text{-}1}$-GaBO} generally display similar performances. 

Table~\ref{tab:all_functions} reports the median and quartiles of the final recommendation, together with Mann-Whitney signed-rank tests to assess whether differences in the distribution of the simple regret of the final recommendation are statistically significant. 
Table~\ref{tab:all_functions} shows that $\alpha$-GaBO consistently ranks among the best-performing approaches with more accurate median values and lower interquartile range. Moreover, even in cases where Euclidean methods achieve slightly better median values, the differences are often not statistically significant, indicating comparable regret distribution with the best-performing Euclidean approach.
App.~\ref{appendix:ablation} ablates the effect of the Riemannian kernel and Riemannian optimization, showing that the kernel leads to the biggest part of the performance gain. Runtimes are reported in App.~\ref{app:runtimes}. 


\begin{table*}[tbp]
\centering
\footnotesize
\caption{Median and interquartile range of the final recommendation, and $p$-values of Mann-Whitney tests on benchmark functions. Statistical tests are against the best BO alternative per dimension (${}^*$: $p<\!0.05$, ${}^{**}$: $p<\!0.01$, ${}^{***}$: $p<\!0.001$).}
\label{tab:all_functions}
\resizebox{.85\linewidth}{!}{
\begin{threeparttable}
\begin{tabular}{cl*{9}{c}}
\toprule
$d$ & Method & \multicolumn{3}{c}{Ackley} & \multicolumn{3}{c}{Rosenbrock} & \multicolumn{3}{c}{Griewank} \\
\cmidrule(lr){3-5} \cmidrule(lr){6-8} \cmidrule(lr){9-11}
 & & Median & IQR & $p$ (vs best) & Median & IQR & $p$ (vs best) & Median & IQR & $p$ (vs best) \\
\midrule
\multirow{4}{*}{2}
& $\mathbb{S}^d$-Eucl. BO & $\bm{0.020}$ & $0.987$ & -- & $0.002$ & $0.008$ & $1.0 \times 10^{-4}$${}^{***}$ & $0.005$ & $0.018$ & $8.1 \times 10^{-1}$ \\
& BORIS & \underline{$0.023$} & $1.793$ & $9.6 \times 10^{-1}$ & $0.006$ & $0.008$ & $1.3 \times 10^{-5}$${}^{***}$ & $0.011$ & $0.025$ & $1.8 \times 10^{-3}$${}^{**}$ \\
& $\alpha_0$-GaBO & $0.033$ & $\bm{0.022}$ & $3.7 \times 10^{-1}$ & {$\bm{0.000}$} & {$\bm{0.000}$} & -- & {$\bm{0.000}$} & \underline{$0.007$} & -- \\
& $\alpha_{-1}$-GaBO & $0.036$ & \underline{$0.030$} & $2.0 \times 10^{-1}$ & \underline{$0.000$} & \underline{$0.001$} & $2.2 \times 10^{-3}$${}^{**}$ & \underline{$0.000$} & {$\bm{0.007}$} & $5.0 \times 10^{-2}$ \\
\cmidrule(lr){2-11}
\multirow{4}{*}{5}
& $\mathbb{S}^d$-Eucl. BO & $0.829$ & $2.044$ & $7.9 \times 10^{-1}$ & \underline{$3.179$} & $5.209$ & $8.4 \times 10^{-2}$ & $0.034$ & $0.048$ & $2.4 \times 10^{-4}$${}^{***}$ \\
& BORIS & $1.634$ & $2.047$ & $4.7 \times 10^{-1}$ & $4.776$ & $5.056$ & $9.7 \times 10^{-4}$${}^{***}$ & $0.031$ & $0.043$ & $2.0 \times 10^{-5}$${}^{***}$ \\
& $\alpha_0$-GaBO & {$\bm{0.261}$} & \underline{$0.323$} & -- & {$\bm{0.078}$} & \underline{$4.451$} & -- & {$\bm{0.000}$} & {$\bm{0.001}$} & -- \\
& $\alpha_{-1}$-GaBO & \underline{$0.275$} & {$\bm{0.245}$} & $4.0 \times 10^{-1}$ & $5.656$ & {$\bm{2.691}$} & $1.8 \times 10^{-5}$${}^{***}$ & \underline{$0.001$} & \underline{$0.003$} & $4.8 \times 10^{-2}$${}^{*}$\\
\cmidrule(lr){2-11}
\multirow{4}{*}{10}
& $\mathbb{S}^d$-Eucl. BO & $0.714$ & \underline{$1.056$} & $2.2 \times 10^{-1}$ & $16.046$ & {$\bm{5.839}$} & $1.4 \times 10^{-3}$${}^{**}$& {$\bm{0.022}$} & $0.040$ & -- \\
& BORIS & $0.866$ & $1.313$ & $8.8 \times 10^{-1}$ & $14.280$ & $17.379$ & $1.6 \times 10^{-1}$ & \underline{$0.029$} & $0.045$ & $5.7 \times 10^{-1}$ \\
& $\alpha_0$-GaBO & {$\bm{0.147}$} & $1.057$ & -- & {$\bm{7.764}$} & $15.942$ & -- & $0.029$ & {$\bm{0.022}$} & $7.9 \times 10^{-1}$ \\
& $\alpha_{-1}$-GaBO & \underline{$0.289$} & {$\bm{0.353}$} & $5.5 \times 10^{-1}$ & \underline{$12.460$} & \underline{$14.683$} & $3.1 \times 10^{-1}$ & $0.034$ & \underline{$0.030$} & $1.4 \times 10^{-1}$ \\
\bottomrule
\end{tabular}
\end{threeparttable}
}
\end{table*}

\subsection{Optimal mixtures of components}

As first real-world application, we evaluate the performance of $\alpha$-GaBO to find optimal mixtures of components. We consider three scenarios: the \textit{concrete compressive strength} dataset~\citep{Yeh1998ModelingOS} from the UCI repository\footnote{https://archive.ics.uci.edu/datasets}~\citep{uci} and two chemical mixtures from the Olympus repository\footnote{https://github.com/aspuru-guzik-group/olympus}~\citep{olympus_2021, olympus_2023}.

\paragraph{Concrete compressive strength.}
This dataset is composed of $1090$ observations with $7$ variables containing the mixture quantities of the components of concrete and $1$ variable containing the number of days from when the mixture was made. The objective is to find the mixture quantities and number of days maximizing the compressive strength of the resulting concrete in MPa. As the parameter space is the product of manifolds $\simplex{6} \times [0, 1]$, we define the $\alpha$-GaBO kernel as the product of the Riemannian squared exponential kernel on $\simplex{6}$ and a $1$-dimensional Euclidean squared exponential kernel and optimize the acquisition function on the product of manifolds. As no oracle is available, we follow the perspective of~\citet{Yeh1998ModelingOS} and use the dataset to train an artificial neural network as an approximation of the real function $F$. We do so with an MLP with $3$ hidden layers, each followed by a dropout layer, and ReLU activations. More detail about the MLP structure are in App.~\ref{appendix:MLPoracle}. 

Fig.~\ref{subfig4:concrete} and Table~\ref{tab:all_real_world} show the performance results. 
We observe that the BO models optimizing on topologically spherical manifolds converge to significantly higher function values with the Euclidean BO slightly outperforming $\alpha_0$-GaBO. This can be explained by the fact that, for this dataset, the optimum tends to lie on the border of the simplex which cannot be reached by \mbox{$\alpha_{\text{-}1}$-GaBO}. 
This result confirms the interest of considering the data topology and geometry in BO.\looseness-1

\paragraph{Olympus.}
We consider two datasets, namely Photo bleaching PCE10 and WF3, which document the photodegradation behavior of polymer blends for organic solar cells under light exposure. Individual data points encode the polymer composition ratios of each blend together with its measured degradation response. Both dataset have $4$ ingredients, $3$ of them are the same (P3HT, PCBM, oIDTBR) while the last is dataset specific (PCE10 or WF3). The objective is to find the mixture of ingredients which minimizes the photodegradation. As the mixture is naturally expressed as a ratio of each component over unity, the probability simplex $\simplex{3}$ is the natural optimization domain for such tasks. For each dataset, we use the oracle provided by the library. 

Fig.~\ref{subfig4:wf3} shows the performance results for the WF3 dataset. We observe that all models perform similarly with the exception of Euclidean BO on the sphere, whose median value is outperformed by all alternatives. For the PCE10 dataset, Fig.~\ref{subfig4:pce10} shows that the $\alpha$-GaBO models exhibit lower function values and significantly lower variances indicating more consistent performance, as confirmed by Table~\ref{tab:all_real_world}.

\newlength{\fullwidth}
\setlength{\fullwidth}{2\columnwidth+\columnsep}

\begin{figure}[t]
    \centering
    \renewcommand{\arraystretch}{1.2} 
    \begin{tabular}{@{}m{0.3cm}@{}c@{}}
        \rotatebox[origin=c]{90}{ \footnotesize $F$} &
        \begin{minipage}[c]{0.3\fullwidth}
            \centering
            \includegraphics[width=\linewidth]{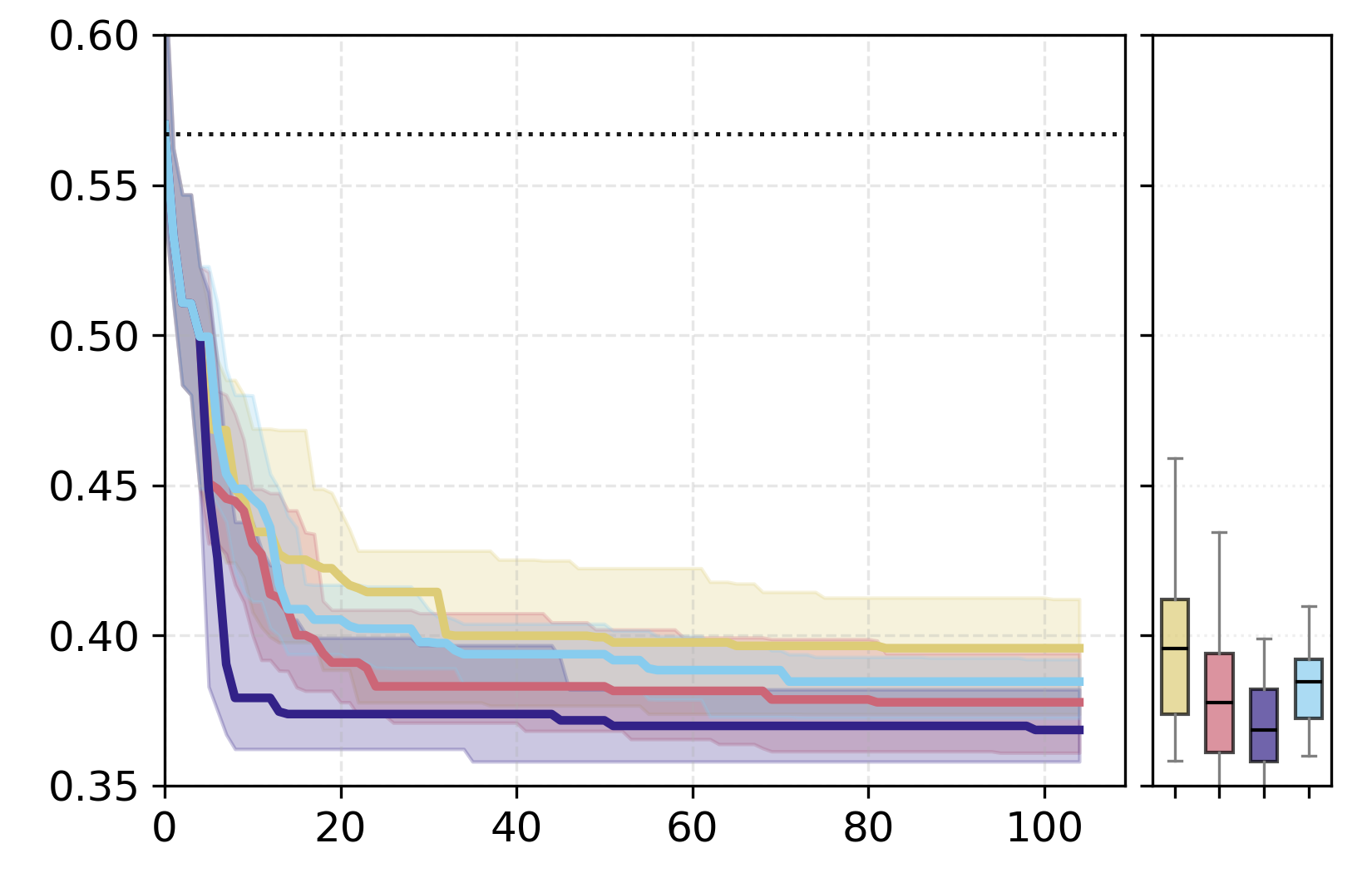} \\
            \vspace{-4pt}
            \footnotesize  $N$
        \end{minipage}
    \end{tabular}
    \caption{Regret (median and quartiles) and distribution over the final recommendation for $\alpha_0$-GaBO (\aogaboline), \mbox{$\alpha_{\text{-}1}$-GaBO} (\amogaboline), $\sphere{d}$-Eucl. BO (\spheuboline), and BORIS (\borisline) on mixture of classifiers for wall-following robot navigation dataset on $\simplex{7}$. All models outperform the best standalone classifier (\blackdottedline).}
    \label{fig:simple_mixture}
\end{figure}

\subsection{Mixture of simple classifiers}

We consider a mixture of classifiers for a non-linearly separable classification task on the UCI \textit{wall-following robot navigation data}. The dataset comprises readings from a circular array of $24$ ultrasound sensors on a SCITOS G5 robot, captured during $4$ full cycles of clockwise wall-following navigation. The objective is to select robot navigation command among $4$ classes (move forward, slight right turn, sharp right turn, slight left turn) from sensory data. We consider $8$ simple classifiers: multinomial logistic, linear SVM, Gaussian na\"{\i}ve Bayes, K-NN, decision trees of depth $1$ (decision stomp) and depth $2$, quadratic discriminant analysis and strongly regularized multinomial logistic regression. The BO objective is to find the optimal mixture of classifiers on $\simplex{7}$ for the robot to optimally navigate.

Fig.~\ref{fig:simple_mixture} and Table~\ref{tab:all_real_world} show performance results. All approaches exhibit similar performances with \mbox{$\alpha_{\text{-}1}$-GaBO} and Euclidean BO on $\sphere{7}$ leading to slightly lower function values, where the latter converges slightly lower and faster.

\begin{table*}[tbp]
\centering
\footnotesize
\caption{Median and interquartile range of the final recommendation, and $p$-values of Mann-Whitney tests on real-world applications. Statistical tests are against the best BO alternative per dimension (${}^*$: $p<\!0.05$, ${}^{**}$: $p<\!0.01$, ${}^{***}$: $p<\!0.001$).  Note that the objective function is maximized for concrete and minimized for the other applications.}
\label{tab:all_real_world}
\resizebox{\linewidth}{!}{
\begin{threeparttable}
\setlength{\tabcolsep}{2.6pt} 
\begin{tabular}{l*{15}{c}}
\toprule
Method
& \multicolumn{3}{c}{Concrete ($d\!=\!7$)}
& \multicolumn{3}{c}{WF3 ($d\!=\!3$)}
& \multicolumn{3}{c}{PCE10 ($d\!=\!3$)}
& \multicolumn{3}{c}{Simple classifiers ($d\!=\!7$)}
& \multicolumn{3}{c}{Robotics ($d\!=\!6$)}
\\
\cmidrule(lr){2-4}
\cmidrule(lr){5-7}
\cmidrule(lr){8-10}
\cmidrule(lr){11-13}
\cmidrule(lr){14-16}
&
Median & IQR & $p$ 
&
Median & IQR & $p$ 
&
Median & IQR & $p$ 
&
Median & IQR & $p$ 
&
Median & IQR & $p$ 
\\
\midrule

$\mathbb{S}^d$-Eucl. BO
& {$\bm{104.098}$} & \underline{$12.172$} & --
& $0.040$ & \underline{$0.017$} & $4.7e^{-1}$
& {$\bm{0.000}$} & $0.024$ & $9.8e^{-1}$
& {$\bm{0.369}$} & \underline{$0.024$} & --
& $2.832$ & $0.357$ & $3.3e^{-3}$${}^{**}$
\\

BORIS
& $68.372$ & $13.874$ & $1.4e^{-9}$${}^{***}$
& {$\bm{0.029}$} & $0.018$ & --
& {$\bm{0.000}$} & $0.026$ & $4.9e^{-1}$
& $0.380$ & {$\bm{0.020}$} & $1.7e^{-2}$${}^{*}$
& $2.778$ & $0.387$ & $2.8e^{-3}$${}^{**}$
\\

$\alpha_0$-GaBO
& \underline{$88.025$} & {$\bm{4.129}$} & $5.0e^{-7}$${}^{***}$
& \underline{$0.031$} & $0.017$ & $1.4e^{-1}$
& \underline{$0.007$} & \underline{$0.005$} & $8.4e^{-4}$${}^{***}$
& $0.396$ & $0.041$ & $1.5e^{-4}$${}^{***}$
& {$\bm{2.634}$} & \underline{$0.225$} & --
\\

$\alpha_{-1}$-GaBO
& $65.760$ & $24.594$ & $1.4e^{-9}$${}^{***}$
& $0.032$ & {$\bm{0.015}$} & $1.6e^{-1}$
& {$\bm{0.000}$} & {$\bm{0.004}$} & --
& \underline{$0.378$} & $0.032$ & $9.9e^{-2}$
& \underline{$2.668$} & {$\bm{0.224}$} & $2.9e^{-1}$
\\

\bottomrule
\end{tabular}
\end{threeparttable}
}
\end{table*}

\subsection{Robotic multi-task control}
\label{subsec:exp-multitask-control}
Finally, we consider a robotic multi-task control problem, where we aim to optimize time-varying priorities assigned to different tasks to maximize the robot performance for a desired complex behavior. 
Given $N$ elementary tasks with torques $\bm{\tau}_i$, the final controller is $\bm{\tau}(\bm{q}, \dot{\bm{q}}, t) = \sum_{i=1}^N \alpha_i(t)\bm{\tau}_i(\bm{q}, \dot{\bm{q}})$, where $\bm{q},\dot{\bm{q}} \in \mathbb{R}^d$ denote the robot joint positions  and velocities, $t \in [0,T]$ is the time, and $\alpha_i(t) \in \mathbb{R}^+$ is the time-dependent $i$-th task weight defining its priority. \citet{Modugno2016,su2018sample} model the task weights as the weighted sum 
$
    \alpha_i(\bm{\pi}_i, t) = s\left( \frac{\sum_{k=1}^K \pi_{i,k} \psi(\mu_k, \sigma_k, t)}{\sum_{k=1}^K \psi(\mu_k, \sigma_k, t)}\right)\mbox{,}
$
where $\psi(\mu_k,\sigma_k,t) = \exp(-\frac{(t-\mu_k)^2}{2\sigma_k^2})$ are radial basis functions with fixed mean $\mu_k$ and variance $\sigma_k$, $s$ is the sigmoid function, and $\bm{\pi}_i \in \mathbb{R}^K_+$ is a set of parameters determining the time-varying task weights. This formulation shifts the problem to finding an optimal matrix of parameters $\Pi = [\bm{\pi}_1^\trsp, \ldots, \bm{\pi}_N^\trsp] \in \mathbb{R}^{N \times K}$ instead of functions $\alpha_i(t)$. 
We propose to model the task weights as elements of the simplex, as ensuring that the weights sum to $1$ guarantee that at least one task is activated at any time~\citep{Jaquier22:SeqBlendSkills}.
Specifically, we model the task weights as 
\begin{equation*}
    \alpha_i(\bm{\pi}_i, t) = \frac{\sum_{k=1}^K \pi_{i,k} \psi(\mu_k,\sigma_k,t)}{\sum_{k=1}^K \psi(\mu_k,\sigma_k,t)} \text{ with } \Pi_{k} \in \Delta^{N-1},
\end{equation*}
where $\Pi_{k}$ denote the $k$-th column of the weight matrix $\Pi \in \mathbb{R}^{N \times K}$, thereby ensuring that $(\alpha_1, \dots, \alpha_N)^\trsp \in \Delta^{N-1}$, see App.~\ref{appendix:softtaskpriorities} for a proof. The parameter space thus becomes the product of manifolds $\prod_{k=1}^K \Delta^{N-1}$. 

We design an experiment where a RB-Y1 humanoid robot placed in front of a cylindrical obstacle similar to a pillar aims at reaching target positions with its two hands while avoiding collisions and sudden potentially-damaging movements. 
We compose $4$ elementary tasks: two reaching tasks for the left and right hands, a posture task to keep the torso straight, and a repulsive force field to avoid collisions. We set $T=30$s and $K=10$ equally-spaced basis functions. The objective is to minimize a loss function penalizing the distance to the target positions, high joint torques, and collisions, see App.~\ref{appendix:softtaskpriorities} for the complete experiment description.

Fig.~\ref{fig:robot} and Table~\ref{tab:all_real_world} show the performance results and snapshots of the optimal robot trajectory. We observe that both $\alpha$-GaBO models outperform their Euclidean counterparts by converging faster to lower function values with significantly lower variance with $\alpha_0$-GaBO performing best overall. The robot follows a collision-free trajectory reaching the desired targets.

\begin{figure}[t]
    \centering
    \setlength{\tabcolsep}{0pt} 
    \renewcommand{\arraystretch}{1.0} 
    \begin{tabular}{m{0.25cm} c c}
        \rotatebox[origin=c]{90}{ \footnotesize $F$}  &
        \begin{minipage}[c]{0.61\linewidth}
            \centering
            \includegraphics[width=\linewidth]{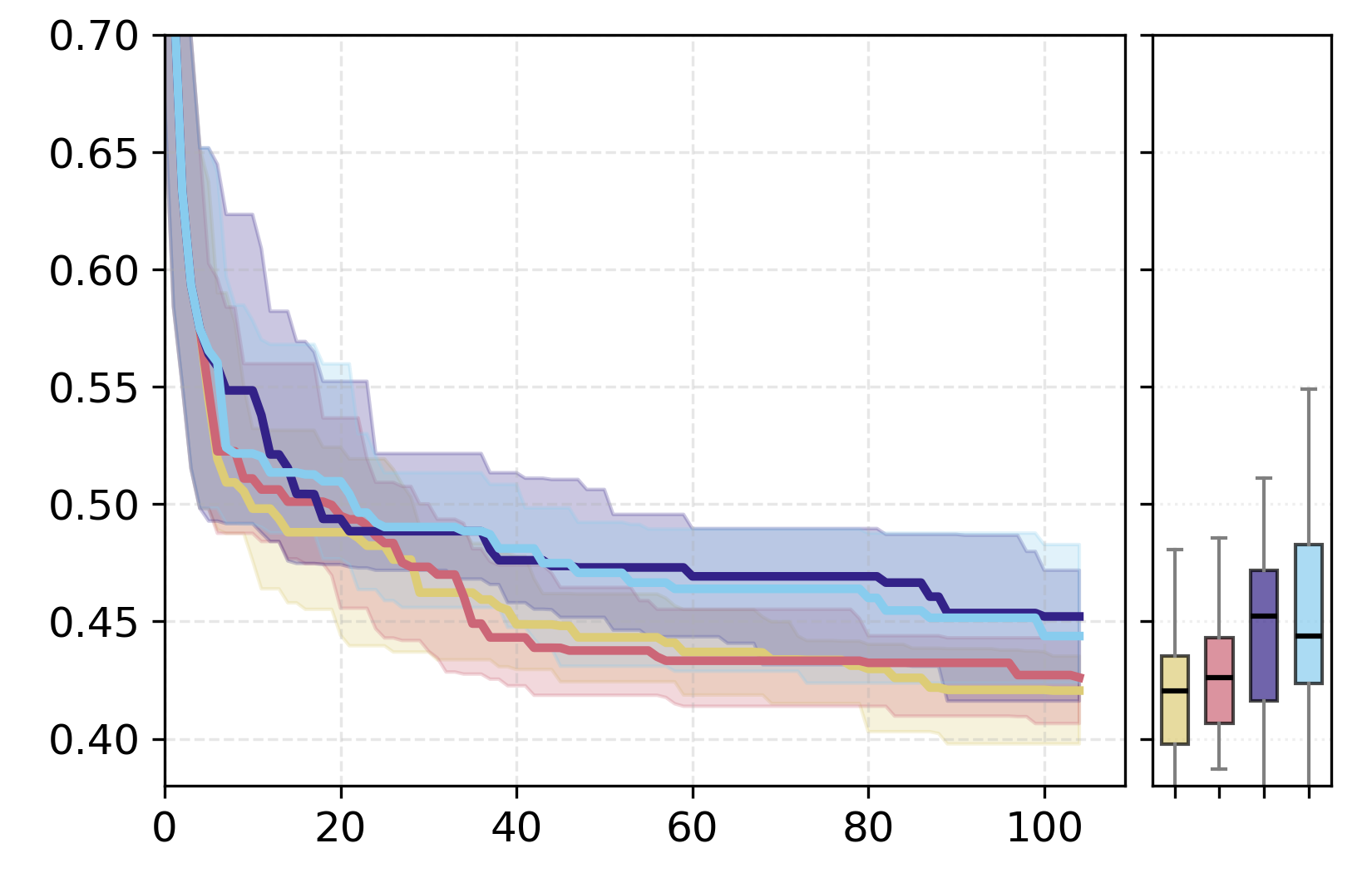} \\
            \vspace{-4pt}
            \footnotesize $N$
        \end{minipage} &
        \begin{minipage}[c]{0.36\linewidth}
            \centering
            \includegraphics[width=\linewidth]{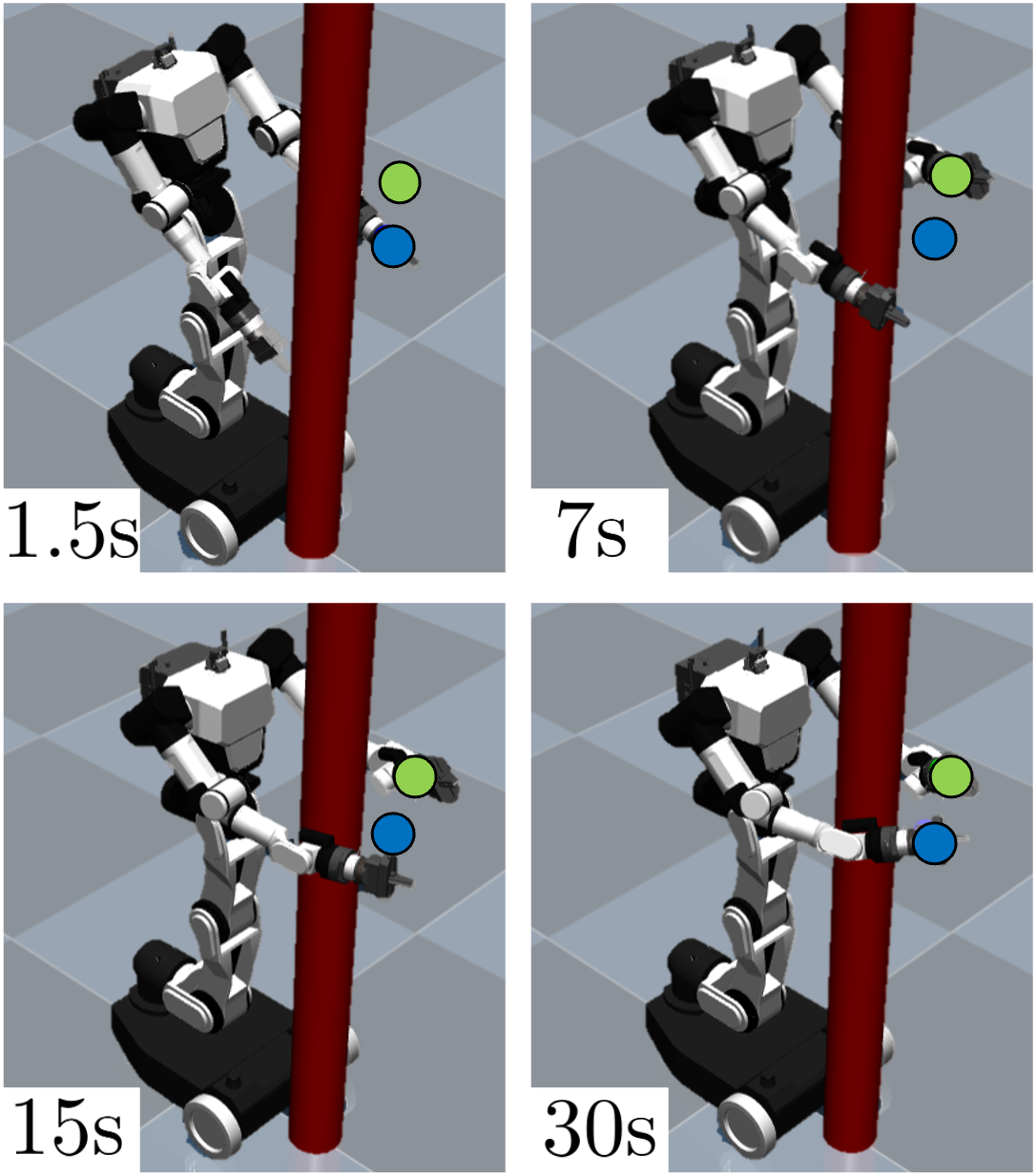}
        \end{minipage} \\[0.0em] 
    \end{tabular}
    \caption{\emph{Left:} Regret (median and quartiles) and distribution over the final recommendation for $\alpha_0$-GaBO (\aogaboline), \mbox{$\alpha_{\text{-}1}$-GaBO} (\amogaboline), $\sphere{d}$-Eucl. BO (\spheuboline), and BORIS (\borisline) for robotic multi-task control. \emph{Right:} Snapshots of the optimal trajectory with target left (\greencircle) and right (\bluecircle) hand positions.}
    \label{fig:robot}
\end{figure}


\section{Conclusions}

We proposed $\alpha$-GaBO, a geometry-aware Bayesian optimization framework designed for objective functions defined on the probability simplex. This is achieved by leveraging information geometry tools to define Matérn kernels on the probability simplex and to design a one-parameter family of acquisition function optimizers accounting for the geometry of the manifold. 
We evaluated $\alpha$-GaBO on synthetic benchmarks and to optimize mixtures of components, mixture of classifiers, and robot task priorities, showcasing superior or matching performances compared to constrained Euclidean methods overlooking geometric priors. 

Besides probabilities and mixtures, the probability simplex has recently been used as a relaxation for categorical data for training flow matching~\citep{davis2024fisher}. In a similar line, we hypothesize that $\alpha$-GaBO may serve as a starting point for designing geometry-aware BO frameworks for categorical domains~\citep{gonzalez2024:surveydiscreteBO}.  
As future work, we also plan to investigate extensions of $\alpha$-GaBO to higher dimensional search spaces via Riemannian latent spaces, akin to~\citep{jaquier20b}.
Another promising direction concerns the design of families of GaBO algorithms based on the $\alpha$ connection for a wider class of information manifolds such as symmetric positive-definite matrices~\citep{amari2016information, ay2017information}.

\clearpage
\begin{acknowledgements} 
    \noindent We greatly acknowledge the \href{https://datalab.unimib.it/}{``DEMS Data Science Lab''} for supporting this work by providing computational resources. This work was partially supported by the Wallenberg AI, Autonomous Systems and Software Program (WASP) funded by the Knut and Alice Wallenberg Foundation. The authors also thank Seungyeon Kim for his help with the implementation of the robotic multi-task control experiment.
\end{acknowledgements}
\bibliography{ref}

\newpage

\onecolumn

\title{Supplementary Material}
\maketitle

\appendix

\section{Probability Simplex}
\label{appendix:IG}

\subsection{Information Geometry of the Probability Simplex}
\label{appendix:IGSimplex}
This Appendix provides a formal extended background on the probability simplex. As the probability simplex is the space of finite probability measures, we start the discussion in an abstract manner involving measures. Moreover, this level of abstractness simplifies the introduction of the Fisher-Rao metric and allows us to refer to an existing result for the extension of the metric to the border of the simplex. Afterwards, we introduce a simple embedding which recovers the standard definition of the simplex as a Euclidean subset. 

We consider a discrete, finite sample space $\Omega = \{\omega_0, \omega_1, \dots , \omega_{d}\}$ with $d+1$ elements, which with no loss of generality we can identify with $I = \{0, 1, \dots,d\} \subseteq \mathbb{N}_0$. We call $\mathcal{F}(\Omega)$ the algebra of real functions on $\Omega$. Its dual $\mathcal{F}^*(\Omega)$ is naturally identified with the space of signed measures on the sample space. Denoting the canonical and the dual basis of these spaces with $\{e_i\}$ and $\{\delta^i\}$, we can express $\mu \in \mathcal{F}^*(\Omega)$ as $\mu = \mu_i \delta^i$ with $\mu_i:=\mu(e_i)$ where we used Einstein summation convention. We define the space of finite probability measures as

\begin{equation}
    \mathcal{P}_+(\Omega) := \{\mu \in \mathcal{F}^*(\Omega)| \mu_i \geq 0, \mu_i \delta^i = 1\},
\end{equation}

and denote its relative interior by $\mathring{\mathcal{P}}_+(\Omega)$, which we can see as an open submanifold of $\mathcal{F}^*(\Omega)$.
Because $\mathcal{F}^*(\Omega)$ is a vector space, it is clear that its tangent bundle is the trivial bundle and the cotangent is the trivialization with $\mathcal{F}(\Omega)$ because of the natural identification $\mathcal{F}^{**}(\Omega)=\mathcal{F}(\Omega)$. Following~\citep{ay2017information, Pistone2019Information}, it is simple to show that 

\begin{align}
    \mathcal{T}_\mu\mathring{\mathcal{P}}_+(\Omega) \cong \mathring{\mathcal{P}}_+(\Omega) \times \mathcal{F}^*_0(\Omega), \\
    \mathcal{T}^*_\mu\mathring{\mathcal{P}}_+(\Omega) \cong \mathring{\mathcal{P}}_+(\Omega) \times \mathcal{F}(\Omega) /\mathbb{R}.
\end{align}

Next, we consider the natural $L^2$ product on $\mathcal{F}(\Omega)$, given by
\begin{align}
    \langle f, g \rangle_\mu = \mu(fg),
\end{align}
with the induced vector space isomorphism to the dual $\mathcal{F}^*(\Omega)$ denoted as $\langle f, . \rangle_\mu$. Using the inverse isomorphism, we can characterize an element of the tangent space at $\mu$ with $f \in \mathcal{F}(\Omega)$ with the Radon-Nikodym derivative as $f = \frac{\text{d}\nu}{\text{d}\mu} = \sum_{i \in I}\frac{\nu_i}{\mu_i}e_i$. Thus, the $L^2$ product can be also expressed as 

\begin{equation}
    \langle f, g \rangle_\mu = \sum_{i \in I} \frac{\nu_i\zeta_i}{\mu_i},
\end{equation}
 
for some $f=\sum\frac{\nu_i}{\mu_i}e_i$,  $g=\sum\frac{\zeta_i}{\mu_i}e_i$, and $\nu, \zeta \in T_\mu\mathring{\mathcal{P}}_+(\Omega)$. Due to the simplicity of this isometry, it is very common to find two equivalent characterizations of the tangent space at $\mu$ in the literature, namely the \textit{measure} characterization where $T_\mu\mathring{\mathcal{P}}_+(\Omega) \cong \mathcal{F}_0^*(\Omega)$, and the \textit{score} characterization where 
$T_\mu\mathring{\mathcal{P}}_+(\Omega) \cong L^2_0(\Omega, \mu)$. In an information theoretic fashion, the inner product defined on such spaces is called the Fisher metric. This characterization of the metric is rather unusual in standard information geometry literature, where it is instead derived from the Fisher information matrix. Our choice is supported by the fact, in the case of discrete probability measures, it is easy to see that both formulations are equivalent~\citep{ay2017information}, whereas Fisher information requires the introduction of some additional concepts we do not need in our approach. We conclude this paragraph on first-order geometry by providing the formula of the gradient of a function $F:\mathring{\mathcal{P}}_+(\Omega) \to \euclideanspace$, also called the natural gradient, 
\begin{equation}
    \operatorname{grad}_\mu F=(DF(\mu) - \mathbb{E}_\mu[DF(\mu)])
\end{equation}
where $D$ denotes the standard Euclidean gradient.

Optimization methods require tools arising from the second order geometry of the space. In information geometry, it is common to work with the conjugate connection structure, which naturally arises from the geometric description of statistical models and allows the construction of a family of connections. The second order geometry is induced by conjugate connections $(\nabla, \nabla^*)$, also called the mixture and exponential connection, respectively. The relationship between these connections is expressed by
\begin{equation}
    Xg(Y,Z) = g(\nabla_X Y,Z) + g(Y, \nabla^*_X Z) \mbox{,}
\end{equation}
where $X,Y,Z$ are vector fields on the information manifold and $g$ is the Fisher-Rao metric. 
We refer the reader to~\citep{amari2016information, ay2017information} for a detailed treatment of this structure which goes beyond the scope of this paper. Next, we report the useful operations used in our method. We use the tangent space characterization $\mathcal{T}_\mu\mathring{\mathcal{P}}_+(\Omega) \cong L^2_0(\Omega)$ and denote tangent vectors as $X, Y \in \mathcal{T}_\mu\mathring{\mathcal{P}}_+(\Omega)$. The mixture and exponential connections then given as

\begin{align}
    \nabla_XY(\mu) &= (\mu, D^2_XY(\mu) + X(\mu)Y(\mu)) \mbox{,} \\
    \nabla^*_XY(\mu) &= \left(\mu, D^2_XY(\mu) +  \mathbb{E}_\mu\left[X(\mu)Y(\mu)\right]\right) \mbox{,}
\end{align}

where $D^2$ is the Euclidean differentiation of a vector field. The associated exponential maps have the form

\begin{align}
    \expmap{\mu}{f} &=\mu + f\mu, \label{eq:app-mixture-connection} \\
    \expmapblank{\mu}^*(f) &= \frac{\operatorname{exp}(f)}{\mathbb{E}_\mu[\operatorname{exp}(f)]}\mu, \label{eq:app-exp-connection} 
\end{align}
respectively, where the exponential is taken component-wise and $(f\mu)_i=f_i\mu_i$. 
The $\alpha$-connection is the one-parameter family of connections defined via the convex combinations of~\eqref{eq:app-mixture-connection}-\eqref{eq:app-exp-connection} as 
\begin{align}
    \nabla^{(\alpha)}_XY(\mu) = (\mu, D^2_XY(\mu) + \frac{1+\alpha}{2}X(\mu)Y(\mu) + \frac{1-\alpha}{2}\mathbb{E}_\mu\left[X(\mu)Y(\mu)\right]) \mbox{,}
\end{align}
with $\alpha \in [-1, 1]$. 
Note that we adopted the convention for which we recover the mixture and exponential connection for $\alpha=1$ and $\alpha=-1$, respectively. The Levi-Civita connection (corresponding to the Fisher-Rao metric) is obtained for $\alpha=0$, which can be rewritten as

\begin{align}
    \nabla^{(\text{LC})} = \frac{\nabla + \nabla^*}{2} \mbox{.}
\end{align}

An essential ingredient in optimization is the Hessian which, for a generic $\alpha \in [-1, 1]$, at certain point $\mu \in \mathring{\mathcal{P}}_+(\Omega)$ takes the form $\mathcal{H}^\alpha_\mu F(X)=\nabla^{(\alpha)}_X\operatorname{grad}_\mu F$ for $X \in \mathcal{T}_\mu\mathring{\mathcal{P}}_+(\Omega)$.

In general, working in function spaces is not easy and requires advanced tools. In this particular setting, it is possible to avoid this by defining the relative interior of the simplex via a parametrization of $\mathring{\mathcal{P}}_+(\Omega)$. Specifically, we define the probability simplex as $\Delta(\Omega) := \{ \bm{x} \in \mathbb{R}^{d+1} | x_i \geq 0 \text{ and } \sum_i x_i = 1\}$. By abuse of notation, we denote this space simply as $\Delta^d$ to emphasize the dimension of the sample space. The relative interior $\mathring{\Delta}^d$ of $\simplex{d}$ can easily be embedded in $\mathring{\mathcal{P}}_+(\Omega)$ via 

\begin{equation}
\label{eq:simplexprobembedding}
    \mathring{\Delta}^d \hookrightarrow \mathring{\mathcal{P}}_+(\Omega) : \bm{x} \mapsto \rho(x_i) \delta^i \mbox{,}
\end{equation}

where $\rho(x_i)$ maps $x_i$ to a constant function. The differentiable structure on $\mathring{\Delta}^d$ is simply obtained by pulling back the structure from  $\mathring{\mathcal{P}}_+(\Omega)$ using this embedding. For completeness with respect to Sec.~\ref{sec:Background}, we also report some useful quantities for Riemannian optimization, namely

\begin{align}
    \operatorname{grad}_{\bm{x} }
F&=DF(\bm{x}) - \bm{x}^T(DF(\bm{x}))\mbox{,} \label{eq:gradient_simplex}\\
\nabla^{(\alpha)}_{\bm{\eta}}\bm{\xi}(\bm{x}) &= \big(\bm{x}, D^2_{\bm{\eta}}\bm{\xi} + \frac{1+\alpha}{2}\bm{\eta} \odot \bm{\xi} + \frac{1-\alpha}{2}\bm{x}^T({\bm{\eta}} \odot \bm{\xi}) \big) \mbox{,} \\
\mathcal{H}^\alpha_{\bm{x}} F(\bm{\eta})&=\nabla^{(\alpha)}_{\bm{\eta}}\operatorname{grad}_{\bm{x}} F
\mbox{,} \\
\operatorname{Exp}^{(1)}_{\bm{x}}(\bm{\eta}) &=\bm{x} + \bm{\eta} \odot \bm{x}\mbox{,} \\
\operatorname{Exp}^{(-1)}_{\bm{x}}(\bm{\eta}) &= \frac{\operatorname{exp}(\bm{\eta})}{\bm{x}^T\operatorname{exp}(\bm{\eta})} \odot \bm{x}\mbox{,} \\
\operatorname{Exp}^{(0)}_{\bm{x}}(\bm{\eta}) &= \big(\sqrt{\bm{x}} \cos (\frac{\| \bm{\eta}\|_F}{2}) + \frac{\sqrt{\bm{x}} \odot \bm{\eta}}{\|\bm{\eta}\|_F} \sin( \frac{\| \bm{\eta}\|_F}{2})\big)^2\mbox{,} 
\end{align}

where $\odot$ is the Hadamard product, $\bm{x} \in \mathring{\Delta}^d$, $F:\mathring{\Delta}^d \to \euclideanspace$ is a differentiable function, $D$ is the Euclidean gradient, $D^2$ the Euclidean vector field differentiation, $\bm{\eta}, \bm{\xi} \in \mathcal{T}_{\bm{x}}\mathring{\Delta}^d$, $\alpha \in [-1, 1]$, $\|\bm{\eta}\|_F\!=\!\sqrt{\bm{x}^T\bm{\eta}^2}$ and the exponentiation and square root are taken component-wise. We note that~\eqref{eq:gradient_simplex} has this form because we characterize tangent vectors as scores, thus differing from, e.g.,~\citep{ay2017information} where a measure characterization is preferred. 

\subsection{Sphere map}
\label{appendix:IGSpheremap}

Until now, we restricted our attention to the interior of the simplex because the inner product is ill defined when approaching the border of the simplex. Note that this is clear with the measure zero characterization of the tangent bundle, but it is easy to see that it is also true for mean zero random variables as the Radon-Nikodym derivative is ill defined. Using a simple observation, we now justify why the manifold structure is well defined also on the boundary. 
We consider the spherical submanifold of $\mathcal{F}(\Omega)$ denoted as $S_{2, +} = \{ f \in \mathcal{F}(\Omega)| f(\omega) >0 \text{ and } \sum_{\omega \in \Omega} f^2(\omega) = 4\}$ and equipped with the inherited structure by restriction in the ambient space. We can now consider the smooth diffeomorphism 

\begin{equation}
    \pi:\mathring{\mathcal{P}}_+(\Omega) \to S_{2, +} \sum_{i \in I} \mu_i \delta^i \mapsto 2 \sum_{i \in I} \sqrt{\mu_i} e_i \mbox{,}
\end{equation}
which is also an isometry. Moreover, as on $S_{2, +}$ the metric can be smoothly extended to the boundary $\partial S_{2, +}$, it can be proven that via the inverse of the mapping, the metric can be extended isometrically also on $\mathcal{P}_+(\Omega)$~\citep[Proposition 2.1]{ay2017information}. By means of the embdedding argument~\eqref{eq:simplexprobembedding}, we can define the following commutative diagram
\[
\begin{tikzcd}
 \mathcal{P}_+(\Omega) \arrow[r, "\pi"] \arrow[d, "\rho^{-1}"'] & \overline{S}_{2,+} \arrow[d, "\rho^{-1}"] \\
\Delta(\Omega) \arrow[r, "\varphi"'] & \sphere{d}_{\geq 0}(2)
\end{tikzcd}
\] 

where we define the sphere map $\varphi := \rho \circ \pi \circ \rho^{-1}$, which can be simply expressed as

\begin{equation}
    \varphi: \Delta(\Omega) \to \sphere{d}_{\geq0}(2): \bm{x} \mapsto 2 \sqrt{\bm{x}} \mbox{,}
\end{equation}

with the square root taken element-wise and $\sphere{d}_{\geq0}(2)$ the closure of the positive orthant of the radius-$2$ sphere w.r.t. the usual metric. Notice that, by a simple rescaling of the metric, it is possible to avoid the multiplication by $2$ and map onto the radius-$1$ sphere, whose extended positive orthant is denoted as $\sphere{d}_{\geq0}$. 

As discuss in Sec.~\ref{sec:alpha-GaBO}, this isometry is particularly useful for optimization on the simplex as we approach the boundary. This because (1) the inner product is numerically well defined on the sphere in contrast to the simplex where it remains intractable and (2) the Levi-Civita connection on the sphere is equivalent to the one on the simplex, while the second order geometry has a simple and well known characterization on the sphere.

\section{Ablation Study: Riemannian vs. Euclidean Kernels and Optimization}
\label{appendix:ablation}

This appendix reports the complete results of the ablation study discussed in Sec.~\ref{subsec:benchmarks}. To disentangle the contributions of the Riemannian kernel and Riemannian optimization of the acquisition function, we evaluate the combinations of Riemannian and Euclidean kernels with Riemannian and Euclidean optimization for optimizing the Griewank function on the $\simplex{5}$. Table~\ref{table:ablation} shows that the use of a Riemannian kernel accounts for the largest part of the performance improvement, while Riemannian optimization provides an additional, albeit smaller, benefit when combined with the Riemannian kernel.

\begin{table}[H]
\centering
\caption{Ablation study on the 5-dimensional simplex with the Griewank objective, comparing Riemannian and Euclidean kernels and optimization methods.}
\label{table:ablation}
\begin{tabular}{@{}lllcc@{}}
\toprule
Manifold & Kernel      & Acq.~Opt.     & Median Regret & IQR  \\
\midrule
Sphere   & Riemannian  & Riemannian    & $-3.728$      & $0.670$\\
Sphere   & Euclidean   & Riemannian    & $-1.473$      & $2.072$\\
Sphere   & Riemannian  & Euclidean     & $-2.880$      & $0.959$\\
Sphere   & Euclidean   & Euclidean     & $-1.480$      &$1.987$\\
\addlinespace
Simplex  & Riemannian  & Riemannian    & $-3.634$      & $0.496$\\
Simplex  & Euclidean   & Riemannian    & $-1.500$      & $1.718$\\
Simplex  & Riemannian  & Euclidean     & $-3.273$      & $0.531$\\
Simplex  & Euclidean   & Euclidean     & $-1.473$      & $0.401$\\
\bottomrule
\end{tabular}
\end{table}

\section{Runtimes}
\label{app:runtimes}

Table~\ref{table:iterationtime} reports the wall-clock time per iteration for the Ackley function. While $\alpha$-GaBO is significantly slower than its Euclidean counterpart in low dimensions, the gap decreases as dimension increases since the number of constraints for the Euclidean optimization of the acquisition function increases.

\begin{table}[h]
\centering
\setlength{\tabcolsep}{4pt} 
\caption{Median (interquartile range) iteration wall-clock times in seconds for the Ackley function. All experiments are run inside a virtual machine with 16~vCPUs AMD EPYC~7V12 with an NVIDIA Tesla~T4.}
\label{table:iterationtime}
\begin{tabular}{lccc}
\toprule
Method  & $d=2$        & $d=5$        & $d=10$       \\
\midrule
$\mathbb{S}^d$-Eucl. BO  & $0.046$ ($0.028$) & $0.068$ ($0.025$) & $0.112$ ($0.008$) \\
BORIS & $0.046$ ($0.024$) & $0.068$ ($0.025$) & $0.112$ ($0.010$) \\
$\alpha_0$-GaBO  & $0.235$ ($0.049$) & $0.241$ ($0.042$) & $0.270$ ($0.033$) \\
$\alpha_{-1}$-GaBO & $0.302$ ($0.099$) & $0.231$ ($0.079$) & $0.245$ ($0.071$) \\
\bottomrule
\end{tabular}
\end{table}

\section{Additional details on the concrete compressive strength MLP oracle}
\label{appendix:MLPoracle}
We build an MLP with $3$ hidden layers, each followed by a dropout layer, and ReLU activations. It is trained with $60\%$ train, $20\% $ validation and $20\%$ test split. We use a batch size of $32$
 and $500$
 epochs with early stopping when validation after $30$
 iterations. Optimization is conducted with an Adam optimizer with learning rate $0.001$
 which is halved on plateau for a maximum of $15$ times. We store weights for the best model on validation loss and check quality of the predictions on the test set.\looseness-1

\section{Additional details on the robotic multi-task control experiment}
\label{appendix:softtaskpriorities}

\subsection{Multi-task control on the probability simplex}

Multi-task control allows the generation of complex robot behavior by combining several tasks. Such controllers prioritized simple tasks either through a strict hierarchy~\citep{khatib2022constraint-consistent}, or by weighting them according to a soft hierarchy~\citep{salini11:softpriority}. The latter corresponds to defining a sequence of tasks with concurrent activations, where smooth transitions are achieved by smoothly varying the weights. While early works manually tuned the task weights to obtained the desired robot behaviors~\citep{salini11:softpriority}, more recent approaches optimize the priorities using covariance matrix adaptation evolution strategy (CMA-ES)~\citep{dehio15:softpriorities,modugno2016:softtaskpriorities} or BO~\citep{su2018sample}.

We denote the robot joint positions, velocities, and accelerations as $\bm{q} \in \mathbb{R}^M$, $\dot{\bm{q}}$, and $\ddot{\bm{q}}$. The standard gravity-compensated robot rigid-body dynamics are~\citep{spong2020robot}
\begin{equation}
    \bm{M}(\bm{q})\bm{\ddot{q}} + \bm{c}(\bm{q}, \dot{\bm{q}}) = \bm{\tau}(\bm{q}, \dot{\bm{q}})\mbox{,}
\end{equation}
where $\bm{M}(\bm{q})$ is the generalized inertia matrix, $\bm{c}(\bm{q}, \dot{\bm{q}})$ is the vector of Coriolis forces, and $\bm{\tau}$ is the external torques, in our case, generated by the multitask controller. 
Given $N$ elementary tasks $\bm{\tau}_i$, a multi-task controller is defined as 
\begin{equation}
\label{eq:multitask-controller-sum}
    \bm{\tau}(\bm{q}, \dot{\bm{q}}, t) = \sum_{i=1}^N \alpha_i(t)\bm{\tau}_i(\bm{q}, \dot{\bm{q}})\mbox{,}
\end{equation}
$t \in \mathbb{R}^+$ is the time, and $\alpha_i(t) \in \mathbb{R}^+$ is the time-dependent $i$-th task weight defining its priority.
Given a set of $N$ simple tasks, we aim to find a set of functions $\alpha_i$ that maximizes an objective function $F$ quantifying the performance of the robot in achieving a complex desired behavior.
\citet{Modugno2016,su2018sample} model the task weights as a weighted sum of normalized radial basis functions, i.e.,
\begin{equation}
\label{eq:multitask-weights-sigmoid}
    \alpha_i(\bm{\pi}_i, t) = s\left( \frac{\sum_{k=1}^K \pi_{i,k} \psi(\mu_k, \sigma_k, t)}{\sum_{k=1}^K \psi(\mu_k, \sigma_k, t)}\right),
\end{equation}
where $\psi(\mu_k,\sigma_k,t) = \exp(-\frac{(t-\mu_k)^2}{2\sigma_k^2})$ are radial basis functions with fixed mean $\mu_k$ and variance $\sigma_k$, $s$ is the sigmoid function, and $\bm{\pi}_i \in \mathbb{R}^K$ is a set of parameters determining the time-varying task weights. This formulation shifts the problem to finding an optimal matrix of parameters $\Pi = [\bm{\pi}_1^\trsp, \ldots, \bm{\pi}_N^\trsp] \in \mathbb{R}^{N \times K}$ instead of functions $\alpha_i(t)$. 
Yet, optimizing $\Pi$ with classical optimization techniques is a complex task as it requires computing or approximating gradients. As an alternative,~\citet{su2018sample} propose to optimize the weights via BO. 

A disadvantage of the formulation~\eqref{eq:multitask-weights-sigmoid} is that it does not guarantee the activation of a task at all times, i.e., the vector $(\alpha_1, \dots, \alpha_N)^\trsp$ may be close or identical to zero, which  in practice results in the robot not performing any task. 
We propose to model the task weights as elements of the probability simplex, i.e., $(\alpha_1, \dots, \alpha_N)^\trsp\in\simplex{N-1}$, as ensuring that the weights sum to $1$ guarantee that at least one task is activated at any time~\citep{Jaquier22:SeqBlendSkills}.
This is achieved by defining each column $\Pi_{k}$ of the matrix $\Pi \in \mathbb{R}^{N \times K}$ as an element of the ${N-1}$-dimensional probability simplex.

\begin{restatable}{proposition}{firstprop}
\label{simplexprop}
    Let a weight matrix $\Pi \in \mathbb{R}^{N \times K}$ with columns $\Pi_{k} \in \Delta^{N-1}$ for each $k$ and define the task weights as

    \begin{equation}
        \alpha_i(\bm{\pi}_i, t) = \frac{\sum_{k=1}^K \pi_{i,k} \psi(\mu_k,\sigma_k,t)}{\sum_{k=1}^K \psi(\mu_k,\sigma_k,t)},
    \end{equation}

    for all $i$. Then $[\alpha_1, \dots, \alpha_N] \in \Delta^{N-1}$. 
\end{restatable}
\begin{proof}
    It is sufficient to prove that, for an arbitrary $\Pi$  with columns $\Pi_{k} \in \Delta^{N-1}$, the sum $\sum_i\alpha_i=1$. Formally,
    \begin{align*}
        \sum_{i=1}^N \alpha_i &= \sum_{i=1}^N \frac{\sum_{k=1}^K \pi_{i,k} \psi(\mu_k,\sigma_k,t)}{\sum_{k=1}^K \psi(\mu_k,\sigma_k,t)} \\
        &=\frac{\sum_{i=1}^N\sum_{k=1}^K \pi_{i,k} \psi(\mu_k,\sigma_k,t)}{\sum_{k=1}^K \psi(\mu_k,\sigma_k,t)} \\
&=\frac{\sum_{k=1}^K\sum_{i=1}^N \pi_{i,k} \psi(\mu_k,\sigma_k,t)}{\sum_{k=1}^K \psi(\mu_k,\sigma_k,t)} \\
&=\frac{\sum_{k=1}^K \psi(\mu_k,\sigma_k,t) \sum_{i=1}^N \pi_{i,k}}{\sum_{k=1}^K \psi(\mu_k,\sigma_k,t)}=1
    \mbox{.}
    \end{align*}
\end{proof}

\subsection{Experimental details}

We design a robotic multi-task control task where a RB-Y1 humanoid robot simulated in a Mujoco environment~\citep{todorov2012:mujoco} is placed in front of a cylindrical obstacle similar to a pillar. The objective of the robot is to reach target positions with its left and right hands while avoiding collisions and sudden, potentially-damaging movements. The controller~\eqref{eq:multitask-controller-sum} is composed of $4$ elementary tasks: 
\begin{itemize}
    \item Tasks 1 and 2: Reaching tasks for the left and right hands defined via proportional derivative (PD) controllers $\bm{\tau}_{1,2}(\bm{q}, \dot{\bm{q}}) = k_p \bm{J}^{\dagger}(\bm{q})(\bm{x} - \bm{x}^*) + k_d\dot{\bm{q}}$, where $\bm{x}$ denotes the hand position obtained via the robot forward kinematics and $\bm{J}^\dagger(\bm{q})$ is the pseudo-inverse of the robot Jacobian. We set the desired left and right hand positions as $\bm{x}^*_{\text{left}}=[0.6, 0.03, 1.33]$, $\bm{x}^*_{\text{right}}=[0.6, -0.02, 1.18]$, $k_p=150$ and $k_d=30$.
    \item Task 3: Posture task to keep the torso straight defined via the PD controllers $\bm{\tau}_{3}(\bm{q}, \dot{\bm{q}}) = k_p(\bm{q} - \bm{q}^*) + k_d\dot{\bm{q}}$ with desired joint configurations $\bm{q}^*=\bm{q}(0)$ equal to the default RB-Y1 rest pose, $k_p=150$ and $k_d=30$.
    \item Task 4: Avoiding obstacle via a repulsive force field. The repulsive force from the obstacle is calculated from the surface of the cylinder following~\citep{khatib1986potential} as $\bm{\tau}(\bm{q})=\eta\frac{(\frac{1}{d} - \frac{1}{d_0})}{d^2} \bm{\nu}$ with $d$ the distance of the robot from the obstacle, $\bm{\nu}$ the surface normal vector, and setting $\eta=3\cdot10^{-5}$, $d_0=0.05$.
\end{itemize}

We define the objective function as a negative loss function, i.e., $F(\Pi) = -\phi(\Pi)$, with the loss
\begin{equation}
    \phi(\Pi)=\sum_i^2 a_i\|\bm{x}_i(T) - \bm{x}_i^*\| + a_3\|\bm{q}(T) - \bm{q}^*\| + \lambda I(\Pi)\mbox{,}
\end{equation}
where $a_i$ the weight for the $i$-th positional task, $\lambda$ is a scalar weight, and $I(\tau)$ is an integral penalization term of the form
\begin{align*}
    I(\Pi)=& w_1 \sum_{i=1}^2 \underbrace{b_i \int_{[0, T]} \|\bm{x}_i(t) - \bm{x}_i^*\| dt + b_3 \int_{[0, T]} \|\bm{q}(t) - \bm{q}^*\| dt }_{\text{trajectory length}} \\ &+ w_2\underbrace{\int_{[0, T]} \bm{\tau}(t)^\trsp \bm{\tau}(t) dt}_{\text{joint torque magnitude}} +  w_3\underbrace{\int_{[0, T]} \chi_{\text{collision}}(t) dt}_{\text{time under collision}} + w_4\underbrace{\int_{[0, T]} \chi_{\text{collision}}(t)F(t) \bm{\nu} dt}_{\text{collision forces}}.
\end{align*}
The collision force is captured by multiplying the force vector $F(t)$ at time $t$ by the obstacle surface normal vector $\bm{\nu}$ pointing inside. The weights $w_j$ are set to consistently rescale the different terms as they can reach a significant order of magnitude and $b_i$ are used to weight different tasks. We set the weights of the fitness function as $a_1=a_2=10$ and $a_3=1$, the weight for the overall integral term penalization as $\lambda=0.01$, the rescaling weights as $w_1=w_3=w_4=1$, $w_2=10^{-3}$, the task weights as $b_1= b_2=1$, $b_3=10$. The integral is calculated using Mujoco's timestep, which is set to $dt=0.008$.

We run each simulation for $35$s and set $T=30$ seconds. We split $[0, T]$ in $K=10$ equally-spaced intervals defining the centers $\mu_k$ of $K=10$ basis functions with $\sigma_k=2$. Therefore, the search space is $\mathcal{X} = \prod_{k=1}^{10} \Delta^{3}$.

\end{document}